\documentclass[letterpaper]{article} 
\usepackage[preprint]{aaai2026}  
\usepackage{times}  
\usepackage{helvet}  
\usepackage{courier}  
\usepackage[hyphens]{url}  
\usepackage{graphicx} 
\usepackage{natbib}  
\usepackage{caption} 
\usepackage{algorithm}
\usepackage{algorithmic}

\usepackage{newfloat}
\usepackage{listings}
\DeclareCaptionStyle{ruled}{labelfont=normalfont,labelsep=colon,strut=off} 
\floatstyle{ruled}
\newfloat{listing}{tb}{lst}{}
\floatname{listing}{Listing}
\usepackage{amsmath}
\usepackage{mathtools}
\usepackage{algorithm}
\usepackage{algorithmic}
\usepackage{multirow}
\usepackage{colortbl}
\usepackage{xcolor}
\usepackage{graphicx}
\usepackage{amssymb}
\usepackage{amsthm}
\usepackage[most]{tcolorbox}
\usepackage{makecell}
\usepackage{booktabs}
\usepackage{rotating}

\newtheorem{theorem}{Theorem}[section]

\theoremstyle{definition}
\newtheorem{definition}[theorem]{Definition}

\theoremstyle{remark}
\newtheorem{remark}[theorem]{Remark}

\title{PIANO: Physics Informed Autoregressive Network}
\author{
    Mayank Nagda, Jephte Abijuru, Phil Ostheimer, Marius Kloft, Sophie Fellenz
}
\affiliations{
    Department of Computer Science\\
    RPTU Kaiserslautern-Landau, Germany
}

\begin{document}

\maketitle
\begin{abstract}
Solving time-dependent partial differential equations (PDEs) is fundamental to modeling critical phenomena across science and engineering. Physics-Informed Neural Networks (PINNs) solve PDEs using deep learning. However, PINNs perform pointwise predictions that neglect the autoregressive property of dynamical systems, leading to instabilities and inaccurate predictions. We introduce Physics-Informed Autoregressive Networks (PIANO)---a framework that redesigns PINNs to model dynamical systems. PIANO operates autoregressively, explicitly conditioning future predictions on the past. It is trained through a self-supervised rollout mechanism while enforcing physical constraints. We present a rigorous theoretical analysis demonstrating that PINNs suffer from temporal instability, while PIANO achieves stability through autoregressive modeling. Extensive experiments on challenging time-dependent PDEs demonstrate that PIANO achieves state-of-the-art performance, significantly improving accuracy and stability over existing methods. We further show that PIANO outperforms existing methods in weather forecasting.

\end{abstract}

\section{Introduction}

\begin{figure*}[t]
    \centering
    \includegraphics[width=0.92\linewidth]{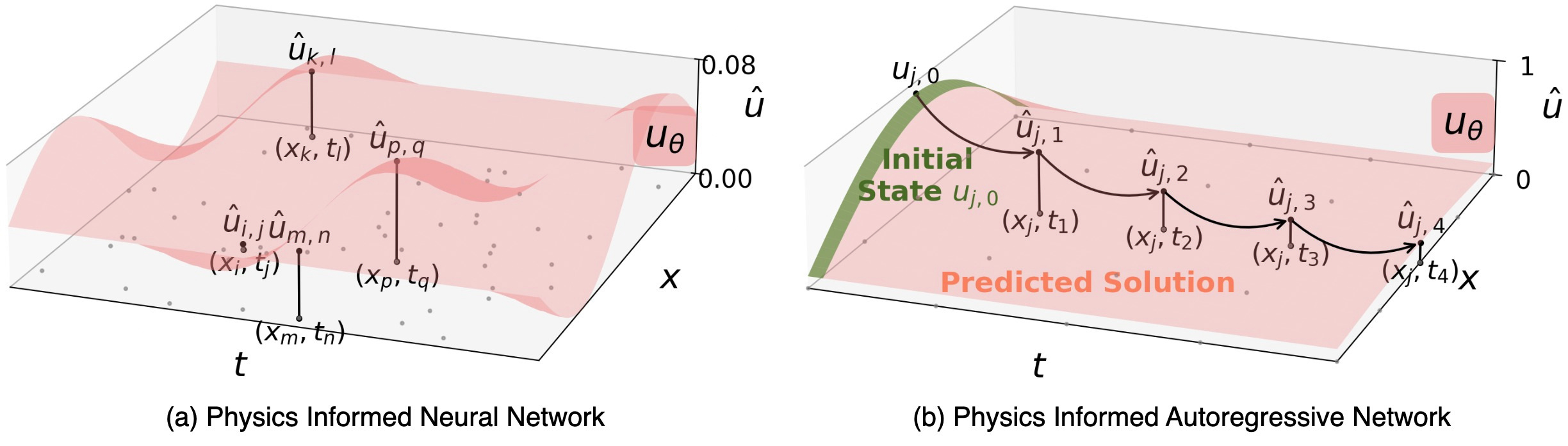}
    \caption{Comparison of (a) standard PINNs and (b) our PIANO framework on solving the Reaction equation. PINNs predict $\hat{u}(x, t)$ independently for each sample $(x, t)$, leading to trivial near-zero solutions. In contrast, PIANO conditions predictions at $t_n$ on the previous state $\hat{u}(x, t_{n-1})$, enabling stable, accurate solution propagation from the known initial state $u(x,0)$.
}
    \label{fig:main-concept}
\end{figure*}

Pierre-Simon Laplace remarked in 1814, \textit{``We may regard the present state of a system as the effect of its past and the cause of its future"} \cite{de1995philosophical}. This observation reflects a core principle in classical physics: the behavior of a dynamical system is determined by its current state, with future states unfolding in an autoregressive manner. Whether predicting the global weather, the rapid dispersion of heat through a solid, or the momentary fluctuations in a fluid's flow, the core challenge lies in understanding how the present state shapes the future. These phenomena are autoregressive and mathematically modeled using time-dependent partial differential equations (PDEs). However, a significant portion of modern machine learning approaches that model dynamical systems by solving PDEs are not autoregressive, i.e., they predict the system’s state at each time step independently, without explicitly conditioning on prior states. As we show in this paper, this can lead to instabilities and inaccurate predictions.

Autoregressive (AR) models are defined by a recursive structure in which the state $u(\cdot, t_n)$ at time \( t_n \) is computed as a function of one or more preceding states: \( u(\cdot, t_n) \coloneqq f(u(\cdot, t_{n-1}), u(\cdot, t_{n-2}), \dots) \). Non-autoregressive (non-AR) models, by contrast, compute the solution independently at each time step from the input coordinates, typically as \( u(\cdot, t_n) \coloneqq f(\cdot, t_n) \). 

Physics-Informed Neural Networks (PINNs) \citep{raissi2019physics} solve time-dependent PDEs by minimizing residuals of the governing equations using neural networks. PINNs are non-AR by design and estimate the state \( u(\cdot, t) \) directly from the coordinates \((\cdot, t)\) in a pointwise fashion. PINNs have been successfully applied to many physical systems, including fluid mechanics \citep{cai2021physics} and cardiovascular flows \citep{raissi2020hidden}, among others. Despite their success, PINNs often struggle to accurately model dynamical systems \citep{wang2024respecting}. To address this, recent extensions have incorporated sequential architectures to better capture temporal dependencies in the input space, as in PINNsFormer \citep{zhao2024pinnsformer} and PINNMamba \citep{xu2025sub}. However, these models remain non-AR, as they are only sequential in the coordinates \((\cdot, t)\) and prediction $u(\cdot, t_n)$ at each time step $t_n$ is made independently of prior state $u(\cdot, t_{n-1})$.

In this work, we introduce PIANO (Physics-Informed Autoregressive Network), an autoregressive physics-informed neural PDE solver. As illustrated in Figure~\ref{fig:main-concept}, PIANO learns a state transition function that is conditioned on its own past predictions. It accurately propagates the known solution from the initial condition forward through time. PIANO offers a robust, stable, and physically plausible framework for simulating dynamical systems. Our contributions are as follows:

\begin{itemize}\item \textbf{Theoretical Foundations.} We present a theoretical analysis on temporal instability in canonical PINNs for time-dependent PDEs and show that PIANO addresses this through an autoregressive formulation ( Section~\ref{sec:theory}).

\item \textbf{A Novel Autoregressive Framework.} We propose PIANO, a physics-informed learning approach that models temporal evolution in dynamical systems by conditioning each prediction on prior states (Section~\ref{sec:main_method}).

\item \textbf{Empirical Validation.} In Section~\ref{sec:exp}, we validate PIANO on challenging time-dependent PDE benchmarks and weather forecasting, demonstrating improved accuracy and stability over existing methods.
\end{itemize}

\section{Related Work}
\paragraph{Physics-Informed Neural Networks (PINNs)}
The concept of using neural networks to solve PDEs dates back several decades \citep{lagaris1998artificial}. The modern incarnation, PINNs \citep{raissi2019physics}, leverages automatic differentiation and deep neural networks to enforce PDE residuals as a soft constraint in the loss function. This has led to widespread applications in diverse scientific domains, including fluid dynamics \citep{cai2021physics}, solid mechanics \citep{haghighat2021physics}, and weather forecasting \citep{verma2024climode} among others. Standard PINNs, typically implemented as pointwise MLPs, map \((x, t)\) directly to \(u(x, t)\) without modeling temporal dependencies, leading to compounding errors in dynamical systems \citep{zhao2024pinnsformer,krishnapriyan2021characterizing}.

\paragraph{Numerical Methods for Solving PDEs}
Classical numerical methods for time-dependent PDEs, such as finite difference, finite element, and finite volume methods, are inherently autoregressive \citep{leveque2007finite,thomas2013numerical}. Methods like Runge-Kutta or Crank-Nicolson explicitly use the solution at time $t_n$ to compute the solution at $t_{n+1}$ \citep{iserles2009first,butcher2016numerical}. This sequential dependency reflects how the state of a system evolves over time through discretized updates defined by the governing equations. 

\paragraph{Improving PINNs for Dynamical Systems}
Addressing the limitations of PINNs in dynamical systems has been an active area of research in recent years. Several strategies have been proposed to enhance the performance of PINNs for time-dependent PDEs. \citet{wang2024respecting} argue that dynamical systems inherently follow an autoregressive structure, and formally demonstrate that standard PINNs violate this principle. They identify this mismatch as a fundamental weakness---one that contributes significantly to the failure of PINNs in accurately modeling temporal evolution. To mitigate this, they propose a causal training algorithm that re-weights the PDE residual loss over time during each iteration, thereby restoring physical causality. Their work calls for new PINN formulations that explicitly respect the autoregressive nature of dynamical systems. Other efforts focus on training strategies such as curriculum learning and sequence-to-sequence scheduling, where models are trained progressively over increasing time horizons \citep{krishnapriyan2021characterizing,penwarden2023unified}. Adaptive techniques have also been proposed, including time-weighted loss functions and dynamic sampling methods \citep{wight2020solving}. \citet{li2024causality} introduce causality-enhanced PINNs, which discretize the PDE loss and leverage transfer learning to adapt model parameters over time. More recently, architectural advances have aimed to capture temporal dependencies more explicitly. These include transformer-based models such as PINNsFormer \citep{zhao2024pinnsformer} and state-space approaches like PINNMamba \citep{xu2025sub}, which incorporate sequential context across collocation points. \citet{wu2024ropinn} proposed Region-Optimized PINNs (RoPINNs), which exploit both temporal and spatial dependencies in the input space to improve predictive accuracy. While these approaches improve temporal consistency, they are not autoregressive. In contrast, we propose a fully autoregressive framework that addresses a critical gap left by prior methods.

\paragraph{Autoregressive Models and Self-Supervision}
Autoregressive models are the backbone of modern sequence modeling, achieving state-of-the-art performance in natural language processing \citep{brown2020language} and time-series forecasting \citep{das2024decoder}. They operate on the principle of predicting the next token/value/state based on a sequence of preceding ones. This powerful autoregressive structure is highly desirable for physical simulations in dynamical systems \citep{wang2024respecting,li2024causality}. However, autoregressive models are typically trained in a supervised manner on large datasets. Our work brings this paradigm to the unsupervised, physics-constrained setting of PINNs. The concept of self-supervised rollout or experience learning is related to techniques in reinforcement learning \citep{hafner2019dream,schrittwieser2020mastering} and generative modeling \cite{ho2020denoising}, where a model learns from its own outputs. PIANO is the first to formulate and solve the challenges of applying autoregressive modeling within a PINN framework.

\section{Theoretically Analyzing Temporal Instability in PINNs}
\label{sec:theory}
This section begins with the preliminaries, followed by a theoretical analysis of error propagation in PINNs for time-dependent PDEs.
\subsection{Preliminaries}
We consider a differential equation defined over a domain \(\Omega \subset \mathbb{R}^d\), with solution \(u: \mathbb{R}^d \to \mathbb{R}^l\). The domain's interior, initial, and boundary subset are denoted by \(\Omega\), \(\Omega_0\), and \(\partial\Omega\), respectively. Differential operators \(\mathcal{O}_\Omega\), \(\mathcal{O}_{\Omega_0}\), and \(\mathcal{O}_{\partial\Omega}\) encode the governing equations (PDEs), initial conditions (ICs), and boundary conditions (BCs). For example, the heat equation is written as \(\mathcal{O}_\Omega(u)(x) = u_t - u_{xx}\). The complete problem formulation is:
\begin{align}
\label{eq:pde_pinn}
\mathcal{O}_\Omega(u)(x) &= 0,\ x \in \Omega; \quad \mathcal{O}_{\Omega_0}(u)(x) = 0,\ x \in \Omega_0;  \\ \notag &\mathcal{O}_{\partial\Omega}(u)(x) = 0,\ x \in \partial\Omega.
\end{align}

PINNs \citep{raissi2019physics} approximate the solution \(u\) with a neural network \(u_\theta\), trained to minimize the residuals of the governing constraints:
\begin{equation}
\label{eq:pinn_loss}
\mathcal{L}(u_\theta) = \sum_{X \in \{\Omega, \Omega_0, \partial\Omega\}} \frac{\lambda_X}{N_X} \sum_{i=1}^{N_X} \left\| \mathcal{O}_X(u_\theta)(x_X^{(i)}) \right\|^2,
\end{equation}
where \(N_X\) is the number of collocation points in subset \(X\), and \(\lambda_X\) weights each term.

\subsection{Uncontrolled Error Propagation in PINNs}

The standard PINN formulation often fails to produce accurate solutions for time-dependent PDEs \cite{zhao2024pinnsformer,xu2025sub}. We argue that this is not simply an optimization issue, but a deeper architectural mismatch. Classical time-stepping schemes like finite difference or Runge--Kutta are explicitly autoregressive: they update the solution at time $ t_{n+1} $ using the known state at $ t_n $, preserving how dynamical systems evolve in time \citep{iserles2009first,butcher2016numerical}. In contrast, PINNs predict each state directly from coordinates $(\cdot,t)$ without conditioning on prior predictions, effectively breaking this autoregressive structure. Viewed through the lens of semigroup theory \citep{pazy2012semigroups}, time-stepping methods approximate an evolution operator that advances the system forward, an operator that PINNs fail to represent. We now formalize this mismatch by defining the evolution operator.

\begin{definition}[Evolution Operator]
\label{eq:evolution_op}
A time-dependent PDE of the form \(\frac{\partial u}{\partial t} = \mathcal{F}(u, t)\) defines a dynamical system. Its solution can be described by an evolution operator, \(\mathcal{G}(\Delta t)\), which maps the state of the system at time \(t_n\) to the state at time \(t_{n+1} = t_n + \Delta t\):
\begin{equation}
    u_{true}(x, t_{n+1}) = \mathcal{G}(\Delta t) [u_{true}(x, t_n)].
\end{equation}
\end{definition}

A stable solver must ensure that errors do not amplify uncontrollably as this operator is applied repeatedly. Let the error of a model \(u_\theta\) at time step \(t_n\) be \(e_n(x) = u_\theta(x, t_n) - u_{true}(x, t_n)\). We can now define the source of instability in non-AR models.

\begin{definition}[One-Step Rollout Error]
\label{eq:one_step_roll}
Given a model's solution \(u_\theta(x, t_n)\), the true physical evolution would yield the state \(\mathcal{G}(\Delta t) [u_\theta(x, t_n)]\) at time \(t_{n+1}\). The one-step rollout error is the discrepancy between the model's actual prediction at \(t_{n+1}\) and the physically evolved state:
\begin{equation}
    \delta_n = \left\| u_\theta(x, t_{n+1}) - \mathcal{G}(\Delta t) [u_\theta(x, t_n)] \right\|_2.
\end{equation}
This error quantifies how poorly the model approximates the true one-step dynamics when initialized from its own prediction at the previous time step.
\end{definition}

\begin{theorem}[Error Propagation in PINNs]
\label{thm:error_prop}
For non-autoregressive PINNs, the error at step \(t_{n+1}\) is bounded by the sum of the propagated error from the previous step \(t_n\) and the one-step rollout error \(\delta_n\):
\begin{equation}
    \|e_{n+1}\|_2 \le L_{\mathcal{G}} \cdot \|e_n\|_2 + \delta_n,
\end{equation}
where \(L_{\mathcal{G}}\) is the Lipschitz constant of the true evolution operator \(\mathcal{G}(\Delta t)\).
\end{theorem}
\begin{proof}[Proof Sketch]
By definition, \(e_{n+1} = u_\theta(x, t_{n+1}) - u_{true}(x, t_{n+1})\). Substituting \(u_{true}(x, t_{n+1}) = \mathcal{G}(\Delta t) [u_{true}(x, t_n)]\) and adding and subtracting \(\mathcal{G}(\Delta t) [u_\theta(x, t_n)]\), we have:
\begin{align*}
    e_{n+1} = & \left( u_\theta(x, t_{n+1}) - \mathcal{G}(\Delta t) [u_\theta(x, t_n)] \right) \\
    & + \left( \mathcal{G}(\Delta t) [u_\theta(x, t_n)] - \mathcal{G}(\Delta t) [u_{true}(x, t_n)] \right).
\end{align*}
Applying the triangle inequality and the Lipschitz property of \(\mathcal{G}\) yields the result. The complete proof is provided in Appendix~\ref{app:proof}
\end{proof}

\begin{remark}
Theorem~\ref{thm:error_prop} reveals a critical flaw in non-autoregressive PINNs. The loss function in Eq.~\eqref{eq:pinn_loss} only penalizes the static PDE residual, leaving the one-step rollout error, \(\delta_n\), unconstrained. This allows a new error to be introduced at each time step, which then compounds with previously propagated errors, causing long-term instability.

\end{remark}

\section{PIANO: The Proposed Framework}
\label{sec:main_method}

We propose PIANO (Physics-Informed Autoregressive Network), a framework that addresses the temporal instability in PINNs by integrating autoregressive modeling into the physics-informed paradigm. This section outlines the PIANO architecture and its self-supervised training strategy.

\subsection{Architecture}
\label{sec:architecture}
\begin{figure*}[t]
  \centering
  \includegraphics[width=\linewidth]{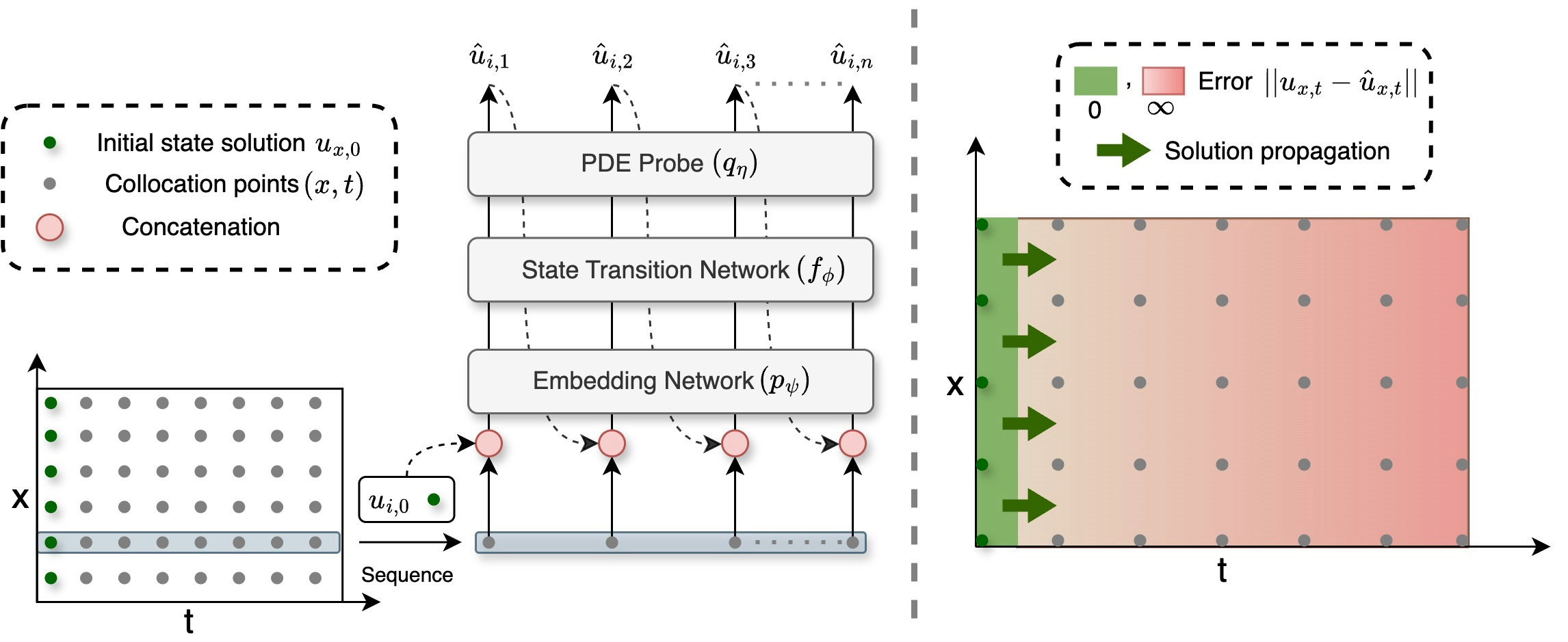}
  \caption{Overview of the proposed method. (Left) The PIANO model processes a sequence of collocation points \( (x_i, t_j)_{j=0}^n \) for each fixed \( x_i \). At each time step \( t_j \), the input—formed by concatenating \( (x_i, t_j) \) with the previous prediction \( \hat{u}(x_i, t_{j-1}) \)—is passed through: an embedding network \( p_\psi \) to produce a high-dimensional embedding, a state transition network \( f_\phi \) to model latent temporal dynamics, and a PDE probe \( q_\eta \) to decode the final solution. (Right) Solution propagation from the initial state. The known initial condition \( u(x, t_0) \) anchors the sequence, enabling stable and accurate learning.
}
  \label{fig:architecture}
\end{figure*}

The architecture is shown in Figure~\ref{fig:architecture}. The domain is defined over coordinates \( (x, t) \in \Omega \subset \mathbb{R}^d \), where \( x \) denotes spatial or physical variables and \( t \) denotes time. The domain is discretized, and for each \( x_i \), a sequence of time points \( (t_j)_{j=0}^M \) is sampled. At each step \( t_j \), the input is constructed by concatenating \( (x_i, t_j) \) with the previous prediction \( \hat{u}(x_i, t_{j-1}) \). The inputs are processed by three components: an Embedding Network, a state-space Transition Network, and a PDE Probe to predict the solution \( \hat{u}(x_i, t_j) \in \mathbb{R}^l \). We now describe each component in detail.
\begin{itemize}
    \item \textbf{Sampling:} For fixed \( x_i \), a sequence of evenly spaced coordinates \( S_{x_i} := [(x_i, t_0), (x_i, t_1), \dots, (x_i, t_M)] \) is sampled from the domain \( \Omega \), where $t_0 \in \Omega_0$ denotes the initial time and $t_M$ corresponds to  the final time step.

    \item \textbf{Input:} At each point \( (x_i, t_j) \in S_{x_i} \), the input vector \( s^i_j := (x_i, t_j, \hat{u}(x_i, t_{j-1})) \in \mathbb{R}^{d+l} \) is formed by concatenating the coordinate with the previous prediction.

    \item \textbf{Embedding Network (\(p_\psi\)):} An embedding network \( p_\psi \) maps each input vector \( s^i_j \) to a higher-dimensional representation \( m^i_j \in \mathbb{R}^k \), enhancing expressiveness and preparing the sequence for further processing.

    \item \textbf{State Transition Network (\( f_\phi \)):} The transition network \( f_\phi \) models latent dynamics by maintaining a hidden state $h$ that summarizes past information. It follows a discrete-time state-space formulation with two steps:

    \[
    \begin{aligned}
    h^i_j &= \sigma(\text{LN}(\mathbf{A}h^i_{j-1} + \mathbf{B}m^i_j)) \\
    o^i_j &= \mathbf{C}h^i_j + \mathbf{D}m^i_j + m^i_j,
    \end{aligned}
    \]
    where \( h^i_j \in \mathbb{R}^k \) is the recurrent hidden state and \( o^i_j \in \mathbb{R}^k \) is the output representation. The function \( f_\phi \) is parameterized by learnable matrices \( \mathbf{A}, \mathbf{B}, \mathbf{C}, \mathbf{D} \in \mathbb{R}^{k \times k} \), along with a nonlinearity \( \sigma \) and layer normalization (LN).

    \item \textbf{PDE Probe (\( q_\eta \)):} Finally, the probe decodes each output:
    \[
    \hat{u}(x_i, t_j) = q_\eta(o^i_j), \quad o^i_j \in \mathbb{R}^k \mapsto \hat{u}(x_i, t_j) \in \mathbb{R}^l.
    \]
    \( q_\eta \) is applied independently at each time step, translating latent features into the PDE solutions.
\end{itemize}

\begin{table*}[t]
\begin{center}
\begin{tabular}{
    l|
    >{\centering\arraybackslash}p{1.2cm}
    >{\centering\arraybackslash}p{1.2cm}
    >{\centering\arraybackslash}p{1.2cm}
    >{\centering\arraybackslash}p{1.2cm}
    >{\centering\arraybackslash}p{1.2cm}
    >{\centering\arraybackslash}p{1.2cm}
    >{\centering\arraybackslash}p{1.2cm}
    >{\centering\arraybackslash}p{1.2cm}
}
\toprule
Model & \multicolumn{2}{c}{Wave} & \multicolumn{2}{c}{Reaction} & \multicolumn{2}{c}{Convection} & \multicolumn{2}{c}{Heat} \\
\cmidrule(lr){2-3} \cmidrule(lr){4-5} \cmidrule(lr){6-7} \cmidrule(lr){8-9}
& rMAE & rRMSE & rMAE & rRMSE & rMAE & rRMSE & rMAE & rRMSE \\
\midrule
PINNs (JCP'19) & 0.4101 & 0.4141 & 0.9803 & 0.9785 & 0.8514 & 0.8989 & 0.8956 & 0.9404\\
QRes (ICDM'21) & 0.5349 & 0.5265 & 0.9826 & 0.9830 & 0.9035 & 0.9245 & 0.8381 & 0.8800\\
FLS (TAI'22) & 0.1020 & 0.1190 & 0.0220 & 0.0390 & 0.1730 & 0.1970 & 0.7491 & 0.7866\\
PINNsFormer (ICLR'24) & 0.3559 & 0.3632 & 0.0146 & 0.0296 & 0.4527 & 0.5217 & 0.2129 & 0.2236\\
RoPINNs (NeurIPS'24) & 0.1650 & 0.1720 & 0.0070 & 0.0170 & 0.6350 & 0.7200 & 0.1545 & 0.1622\\
KAN (ICLR'25) & 0.1433 & 0.1458 & 0.0166 & 0.0343 & 0.6049 & 0.6587 & 0.0901 & 0.1042\\
PINNMamba (ICML'25) & 0.0197 & 0.0199 & 0.0094 & 0.0217 & 0.0188 & 0.0201 & 0.0535 & 0.0583\\
PIANO (ours) & \textbf{0.0057}  & \textbf{0.0059}  & \textbf{0.0001} & \textbf{0.0008} & \textbf{0.0032} & \textbf{0.0104} & \textbf{0.0000} & \textbf{0.0002} \\
\bottomrule
\end{tabular}
\end{center}
\caption{PIANO as a robust and accurate PDE solver across a range of PDE benchmarks. rMAE and rRMSE scores are reported separately for each PDE. Best values are highlighted in bold. PIANO outperforms baselines across all benchmarks.}
\label{tab:overall}
\end{table*}

\subsection{Physics-Informed Experience Learning}

We propose Physics-Informed Experience Learning (PIEL), a training paradigm where the model improves by enforcing physical consistency over its own rollouts. Starting from the known initial condition \( u(x, t_0) \), PIANO autoregressively predicts the trajectory by conditioning each new state on the previous one, with gradients propagated through time via backpropagation through time (BPTT).

For each spatial location \( x_i \), the rollout is:
\[
\hat{u}(x_i, t_j) = u_\theta(x_i, t_j, \hat{u}(x_i, t_{j-1})), \quad j=1,\dots,M,
\]
where \( \theta = \{\psi, \phi, \eta\} \) denotes model parameters.

The computational graph is discrete in time and not connected across neighboring spatial points. Hence, we approximate all PDE derivatives using second-order accurate finite differences applied over the full predicted solution grid. 

For region \( X \in \{\Omega, \partial\Omega\} \), we define the residual energy:
\begin{equation}
    \mathcal{E}_X(x_i, u_\theta) = \frac{1}{M} \sum_{j=1}^{M} \left\| \mathcal{O}_X[\hat{u}](x_i, t_j) \right\|^2,
\end{equation}
where \(\mathcal{O}_X[\hat{u}]\) denotes the PDE residual (via second-order finite differences) for \(X=\Omega\) and the boundary condition error for \(X=\partial\Omega\).

The total training loss aggregates these contributions:
\begin{equation}
\label{eq:piano_loss_energy}
    \mathcal{L}_{\text{PIANO}} = \sum_{X \in \{\Omega, \partial\Omega\}} \frac{\lambda_X}{N_X} \sum_{i=1}^{N_X} \mathcal{E}_X(x_i, u_\theta),
\end{equation}
where \( \lambda_X \) weights PDE and boundary contributions. When data is available, teacher forcing with a data loss can be added, allowing PIEL to operate in both white-box (physics-only) and grey-box (physics+data) settings. Appendix~\ref{app:training} provides the full PIEL training algorithm for completeness.

\begin{remark}[Temporal Stability with PIEL]
By minimizing \(\mathcal{L}_{\text{PIANO}}\), \( u_\theta \) is trained to reduce the one-step rollout error \(\delta_n\) (Theorem~\ref{thm:error_prop}), preventing uncontrolled error accumulation and promoting stable long-term predictions.
\end{remark}

\section{Experiments}
\label{sec:exp}

To empirically demonstrate the effectiveness of PIANO, we evaluate its performance on PDE benchmarks (Section~\ref{sec:pde_tasks}) and a real-world weather forecasting task (Section~\ref{sec:real_world_tasks}).

\subsection{White-box PDE benchmarks}
\label{sec:pde_tasks}

\paragraph{PDE Benchmarks} We evaluate PIANO on four time-dependent PDE benchmarks: the Wave, Reaction, Convection, and Heat equations. These benchmarks are widely used in the literature \citep{zhao2024pinnsformer,wu2024ropinn,xu2025sub} and span diverse numerical challenges: higher-order derivatives (Wave), nonlinear dynamics (Reaction), numerical stiffness (Heat), and transport-dominated behavior prone to numerical diffusion (Convection). Detailed description of each PDE is provided in Appendix~\ref{app:pde_setup}.

\paragraph{Baselines} 
We benchmark PIANO against a broad set of baselines: classical PINN variants (MLP-based PINNs \citep{raissi2019physics}, First-Layer Sine networks (FLS) \citep{wong2022learning}, and Quadratic Residual Networks (QRes) \citep{bu2021quadratic}); recent advances (Kolmogorov–Arnold Networks (KANs) \citep{liu2025kan} and Region-Optimized PINNs (RoPINNs) \citep{wu2024ropinn}); and state-of-the-art sequential models (PINNsFormer \citep{zhao2024pinnsformer} and PINNMamba \citep{xu2025sub}), which are sequential but not autoregressive—ideal for testing PIANO’s autoregressive advantage. Full baseline descriptions are provided in Appendix~\ref{app:baselines}.

\paragraph{Implementation Details} 
PIANO is implemented as a state-space architecture (Section~\ref{sec:main_method}) and trained on a $200 \times 200$ discretized spatio-temporal grid using the AdamW optimizer. All models are trained on approx. same number of samples. For fairness, baselines rely on official implementations with reported hyperparameters and training routines, and we report their published results when available; significance testing was not performed as baseline papers do not report variance.  Regarding computational cost, PIANO is comparable to sequential baselines like PINNsFormer and PINNMamba, while its simpler state-space architecture offers potential efficiency gains. Performance is measured using relative Mean Absolute Error (rMAE) and relative Root Mean Squared Error (rRMSE), which are standard metrics in the PINN literature \cite{xu2025sub}. Full detail on hyperparameters and implementation are provided in Appendix~\ref{app:hyperparameters}.

\paragraph{Quantitative Evaluation} The quantitative results from the PDE benchmarks are presented in Table~\ref{tab:overall}. The analysis reveals a clear performance hierarchy: pointwise PINNs struggle with high errors (validating Theorem~\ref{thm:error_prop}), while sequential models like PINNMamba show improvement, yet PIANO consistently establishes a new state-of-the-art across all four benchmarks. It substantially outperforms all baselines, often reducing the relative error by more than an order of magnitude. For the Reaction and Heat equations, our model's error is reduced to near-zero, highlighting the exceptional stability of the autoregressive approach for non-linear and diffusive systems. Crucially, it also reduces the error of the strongest sequence-based model (PINNMamba) by about threefold on the challenging Wave and Convection equations, demonstrating that a true autoregressive structure is more robust and accurate than mere sequential input processing. These results confirm that PIANO is an accurate and broadly applicable method for solving diverse time-dependent dynamical systems.

\begin{figure*}[t]
  \centering
  \includegraphics[width=\linewidth]{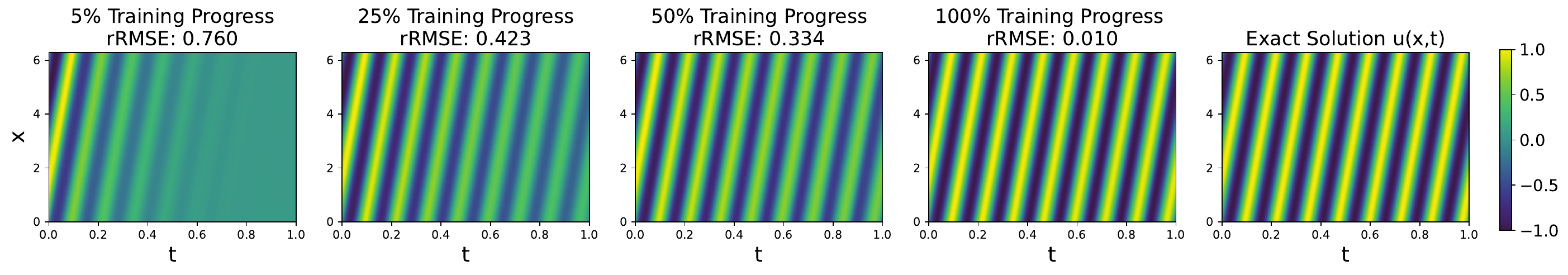}
  \caption{Visualization of PIANO's training dynamics on the convection equation. Each subplot shows the predicted spacetime solution $u(x, t)$ at a given training stage. Early in training (5\%), the solution near the initial condition ($t=0$) is correct, but errors dominate the rest of the domain. As training progresses (25-50\%), PIANO improves its ability to propagate this initial information forward in time, and by the end of training, the prediction becomes indistinguishable from the exact solution, demonstrating the effectiveness of its autoregressive formulation in learning stable temporal propagation.
}
  \label{fig:training_dynamics}
\end{figure*}

\paragraph{Qualitative Evaluation} Figure~\ref{fig:training_dynamics} provides a qualitative visualization of PIANO's training dynamics on the convection equation. The sequence of subplots illustrates how the model's predicted solution evolves at various stages of training. The figure shows how the model progressively learns to propagate the known initial condition forward in time. At 5\% training progress, the prediction near the initial state (\( t=0 \)) is correct by construction, but the model fails to extend that information through the domain, leading to high error (rRMSE 0.760). By 25–50\%, PIANO begins to propagate the initial state more accurately—the diagonal wave structure appears and sharpens as training continues. At 100\%, the model fully captures the transport dynamics, producing a solution visually indistinguishable from the ground truth (rRMSE 0.010). This progression illustrates how PIANO’s autoregressive framework incrementally improves the fidelity of each propagated step, converging from a partially informed state into a globally consistent and physically accurate solution. Additional qualitative results are presented in Appendix~\ref{app:qualitative}.

\subsection{Global Weather Forecasting}
\label{sec:real_world_tasks}
\paragraph{Background}
Weather forecasting has traditionally been dominated by numerical simulations of complex atmospheric physics. Although powerful, these methods are computationally demanding. Recently, deep learning models have emerged as a promising alternative, yet they often function as a ``black-box" that neglect the underlying physical principles. A more robust approach involves integrating physical laws with deep learning approaches. ClimODE~\citep{verma2024climode} is a recent model that successfully applies the physics-informed strategy to weather forecasting. It is built on a core principle from statistical mechanics: weather evolution can be described as a continuous-time advection process, which models the spatial movement and redistribution of quantities like temperature and pressure. By framing the problem as a neural Ordinary Differential Equation (ODE) that adheres to the advection equation, ClimODE enforces value-conserving dynamics, a strong inductive bias that leads to more stable and physically plausible forecasts. With PIANO, we build on the ClimODE framework by introducing an autoregressive training scheme to further enhance predictive accuracy.

\begin{figure*}[t]
  \centering
  \includegraphics[width=\linewidth]{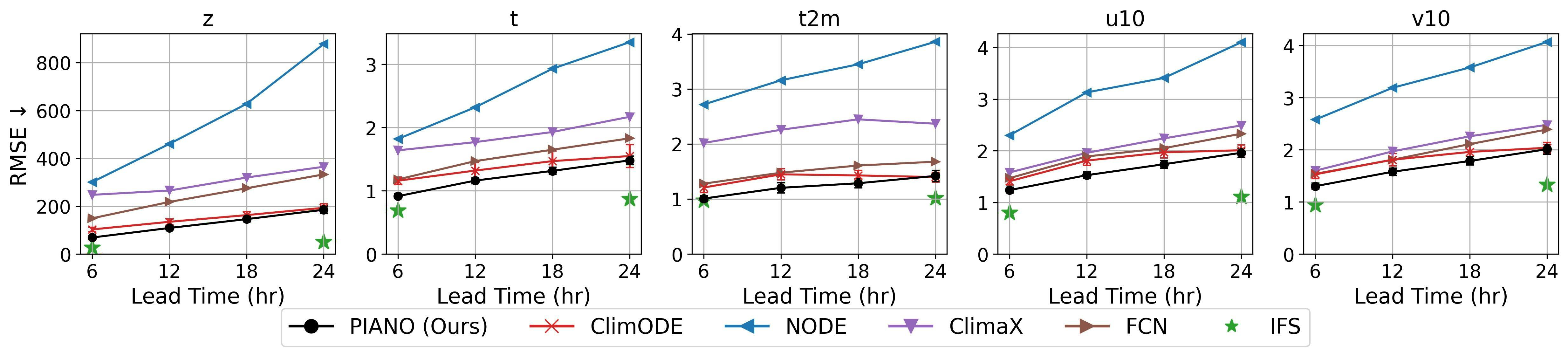}
  \caption{Comparison of global weather forecasting performance on the ERA5 dataset. The plot shows the latitude-weighted RMSE for PIANO against recent baselines over lead times from 6 to 24 hours. The task involves predicting five key atmospheric variables: geopotential (\texttt{z}), atmospheric temperature (\texttt{t}), 2-meter surface temperature (\texttt{t2m}), and the 10-meter U-wind (\texttt{u10}) and V-wind (\texttt{v10}) components. PIANO consistently achieves lower error across all variables, showcasing the effectiveness of its autoregressive physics-informed framework for modeling complex dynamical systems.
}
  \label{fig:forecast_comparison}
\end{figure*}

\paragraph{Setup}
Weather forecasting involves predicting the evolution of key atmospheric variables such as atmospheric temperature ($\texttt{t}$), surface temperature ($\texttt{t2m}$), horizontal wind components ($\texttt{u10}$, $\texttt{v10}$), and geopotential ($\texttt{z}$). We adopt the physics-informed framework of ClimODE~\citep{verma2024climode}, which models weather evolution as a continuous-time process governed by a system of neural ODEs. This system jointly evolves the weather state, denoted by $u(t)$, and a corresponding velocity field, $v(t)$.

The ODE system has two components. The first governs the rate of change of the weather state, $\dot{u}$, and is constrained by the physical advection equation, which ensures that quantities are transported and conserved according to physical principles. The second component governs the rate of change of the velocity field, $\dot{v}$, which is learned by a neural network, $f_\theta$. This network takes as input the current state $u(\tau)$, its spatial gradient $\nabla u(\tau)$, the velocity $v(\tau)$, and spatiotemporal embeddings $\psi$ to determine the acceleration of the flow.

In PIANO, we use this same physics-informed ODE structure but introduce an autoregressive training strategy with teacher forcing to reduce error accumulation. Instead of forecasting the entire trajectory in one step, the model is conditioned on the ground truth from the previous time point. The revised forecast equation for a single time step from $t_i$ to $t_{j}$ is given by:
\begin{equation*}
\begin{bmatrix} \hat{u}(t_{j}) \\ \hat{v}(t_{j}) \end{bmatrix} = \begin{bmatrix} y_i \\ v(t_i) \end{bmatrix} + \int_{t_i}^{t_{j}} \begin{bmatrix} -\nabla \cdot (\hat{u}(\tau)\hat{v}(\tau)) \\ f_{\theta}(\hat{u}(\tau), \nabla \hat{u}(\tau), \hat{v}(\tau), \psi) \end{bmatrix} d\tau,
\end{equation*}
where $y_i$ denotes the observed ground truth state at time $t_i$, $v(t_i)$ is the inferred velocity at that time, and $\nabla \cdot$ is the spatial divergence operator.

We evaluate PIANO on the ERA5 dataset~\cite{rasp2020weatherbench}, a benchmark for global weather forecasting providing 6-hourly reanalysis data at 5.625$^\circ$ resolution for five variables: \texttt{t}, \texttt{t2m}, \texttt{u10}, \texttt{v10}, and \texttt{z}. We compare against several state-of-the-art baselines including Neural ODE (NODE)~\cite{verma2024climode}, FCN~\cite{pathak2022fourcastnet}, ClimaX~\cite{nguyen2023climax}, and the original ClimODE~\cite{verma2024climode}. As a reference, we also report results for the gold-standard Integrated Forecasting System (IFS) \cite{81367} which is one of the most advanced global physics simulation model and has high computational demands. Performance is evaluated using two standard metrics: root mean square error (RMSE) and anomaly correlation coefficient (ACC). RMSE quantifies the absolute prediction error, while ACC measures the correlation between predicted and observed anomalies, capturing the directional accuracy. Both metrics are latitude-weighted to reflect the spherical geometry of the Earth.

\paragraph{Results}
Figure~\ref{fig:forecast_comparison} compares PIANO against baselines on RMSE across five key atmospheric variables over 6 to 24-hour lead times. PIANO consistently achieves a lower RMSE across all variables and horizons, indicating improved forecast accuracy. Notably, its performance gains are most significant at shorter lead times, where the autoregressive use of observed initial states effectively limits error propagation. These results establish PIANO as the state-of-the-art among physics-informed models for ERA5 forecasting, highlighting the benefits of combining an autoregressive training strategy with a physics-informed framework for simulating complex dynamical systems. The complete result table with ACC is provided in Appendix~\ref{app:weather_forecasting}.

\section{Ablation Studies}
We perform ablation studies on the 1D Reaction equation to evaluate the contributions of PIANO’s two key components: its high-order finite difference scheme for computing derivatives and the autoregressive architecture. All reported results are averaged over ten independent runs to ensure statistical reliability. Table~\ref{tab:ablation} summarizes the findings.  

\subsubsection{Finite Difference Scheme}  
PIANO employs finite differences (FD) instead of automatic differentiation (AD) to approximate PDE derivatives. This choice provides stable gradients, lower memory requirements, and allows training with modern first-order optimizers such as AdamW~\cite{kingma2014adam}. In contrast, AD-based PINNs typically rely on quasi-second-order optimizers such as L-BFGS, which are less scalable and poorly suited for stochastic mini-batch training.  

To assess the importance of derivative precision, we compare a second-order FD scheme against a first-order version. The second-order scheme achieves an rRMSE of \( 0.0008 \), while the first-order scheme yields \( 0.0174 \). This difference is statistically significant (two-tailed t-test, \( p < 0.05 \)), demonstrating that higher-order derivative approximations are essential for accurate physics-informed training. Notably, even the first-order version of PIANO performs better than most baselines in Table~\ref{tab:overall}, underscoring the robustness of the approach.  

\subsubsection{Autoregressive Backbone}  
Using the validated second-order FD scheme, we now examine PIANO’s autoregressive architecture. A non-autoregressive baseline (“PIANO (Non-AR)”), which functions like a standard PINN, performs poorly with an rRMSE of \( 0.8010 \), confirming that autoregressive formulation is critical. Introducing progressively stronger recurrent backbones significantly improves performance: an MLP-based variant reaches \( 0.0502 \), a GRU-based model achieves \( 0.0061 \), and the full state-space architecture attains the lowest error of \( 0.0008 \). All improvements are statistically significant (\( p < 0.05 \)).  

These results demonstrate that PIANO’s autoregressive design is fundamental to its accuracy and stability. Even a basic autoregressive MLP variant dramatically outperforms the non-autoregressive baseline, and the full state-space backbone achieves near-zero error, validating its role as the most effective temporal modeling choice.

\begin{table}[t]
\centering
\begin{tabular}{lc}
\toprule
\textbf{Method} & \textbf{rRMSE ($\downarrow$)} \\
\midrule
\multicolumn{2}{c}{\textbf{Finite Difference Schemes}} \\
\midrule
First Order Accurate & 0.0174 $\pm$ 0.0195 \\
Second Order Accurate & 0.0008 $\pm$ 0.0001 \\
\midrule
\multicolumn{2}{c}{\textbf{Autoregressive Backbone}} \\
\midrule
PIANO (Non-AR) & $0.8010 \pm 0.1293$ \\
PIANO (MLP) & 0.0502 $\pm$ 0.0017 \\
PIANO (GRU) & 0.0061 $\pm$ 0.0009 \\
PIANO (SSM) & 0.0008 $\pm$ 0.0001 \\
\bottomrule
\end{tabular}
\caption{Ablation study on the 1D Reaction equation (mean over ten runs). Results isolate the contributions of PIANO’s autoregressive backbone and the precision of the finite difference scheme}
\label{tab:ablation}
\end{table}

\section{Conclusion}
We present PIANO, a physics-informed autoregressive framework for solving time-dependent PDEs. Our theoretical analysis demonstrates that non-autoregressive PINN formulations are unstable and accumulate errors. By aligning model design with the autoregressive property of dynamical systems, PIANO mitigates the error accumulation seen in conventional PINNs and provides a stable foundation for learning physical dynamics. Experiments on a challenging PDE benchmark and global weather forecasting demonstrate that PIANO achieves state-of-the-art accuracy. Beyond these gains, PIANO points to a broader direction: physics-informed learning can benefit significantly from architectures that respect the temporal evolution of the systems they model. Extending this approach to multi-scale processes and real-world scientific applications presents an exciting path for future research.

\section*{Acknowledgements}
Part of this work was conducted within the DFG Research Unit FOR 5359 on Deep Learning on Sparse Chemical Process Data (BU 4042/2-1, KL 2698/6-1, and KL 2698/7-1). MK and SF further acknowledge support by the DFG TRR 375 (ID 511263698), the DFG SPP 2298 (KL 2698/5-2), and the DFG SPP 2331 (FE 2282/1-2, FE 2282/6-1, and KL 2698/11-1). Additional support was provided by the Carl-Zeiss Foundation within the initiatives AI-Care and Process Engineering 4.0, as well as the BMBF award 01\textbar S2407A.

\bibliography{aaai2026}

\begin{thebibliography}{33}
\providecommand{\natexlab}[1]{#1}

\bibitem[{Brown et~al.(2020)Brown, Mann, Ryder, Subbiah, Kaplan, Dhariwal, Neelakantan, Shyam, Sastry, Askell et~al.}]{brown2020language}
Brown, T.; Mann, B.; Ryder, N.; Subbiah, M.; Kaplan, J.~D.; Dhariwal, P.; Neelakantan, A.; Shyam, P.; Sastry, G.; Askell, A.; et~al. 2020.
\newblock Language models are few-shot learners.
\newblock \emph{Advances in neural information processing systems}, 33: 1877--1901.

\bibitem[{Bu and Karpatne(2021)}]{bu2021quadratic}
Bu, J.; and Karpatne, A. 2021.
\newblock Quadratic residual networks: A new class of neural networks for solving forward and inverse problems in physics involving pdes.
\newblock In \emph{Proceedings of the 2021 SIAM International Conference on Data Mining (SDM)}, 675--683. SIAM.

\bibitem[{Butcher(2016)}]{butcher2016numerical}
Butcher, J.~C. 2016.
\newblock \emph{Numerical methods for ordinary differential equations}.
\newblock John Wiley \& Sons.

\bibitem[{Cai et~al.(2021)Cai, Mao, Wang, Yin, and Karniadakis}]{cai2021physics}
Cai, S.; Mao, Z.; Wang, Z.; Yin, M.; and Karniadakis, G.~E. 2021.
\newblock Physics-informed neural networks (PINNs) for fluid mechanics: A review.
\newblock \emph{Acta Mechanica Sinica}, 37(12): 1727--1738.

\bibitem[{Das et~al.(2024)Das, Kong, Sen, and Zhou}]{das2024decoder}
Das, A.; Kong, W.; Sen, R.; and Zhou, Y. 2024.
\newblock A decoder-only foundation model for time-series forecasting.
\newblock In \emph{Forty-first International Conference on Machine Learning}.

\bibitem[{De~Laplace(1995)}]{de1995philosophical}
De~Laplace, M. 1995.
\newblock \emph{A philosophical essay on probabilities}.
\newblock Courier Corporation.

\bibitem[{ECMWF(2023)}]{81367}
ECMWF. 2023.
\newblock \emph{IFS Documentation CY48R1 - Part I: Observations}.
\newblock 1. ECMWF.

\bibitem[{Hafner et~al.(2019)Hafner, Lillicrap, Ba, and Norouzi}]{hafner2019dream}
Hafner, D.; Lillicrap, T.; Ba, J.; and Norouzi, M. 2019.
\newblock Dream to control: Learning behaviors by latent imagination.
\newblock \emph{arXiv preprint arXiv:1912.01603}.

\bibitem[{Haghighat et~al.(2021)Haghighat, Raissi, Moure, Gomez, and Juanes}]{haghighat2021physics}
Haghighat, E.; Raissi, M.; Moure, A.; Gomez, H.; and Juanes, R. 2021.
\newblock A physics-informed deep learning framework for inversion and surrogate modeling in solid mechanics.
\newblock \emph{Computer Methods in Applied Mechanics and Engineering}, 379: 113741.

\bibitem[{Ho, Jain, and Abbeel(2020)}]{ho2020denoising}
Ho, J.; Jain, A.; and Abbeel, P. 2020.
\newblock Denoising diffusion probabilistic models.
\newblock \emph{Advances in neural information processing systems}, 33: 6840--6851.

\bibitem[{Iserles(2009)}]{iserles2009first}
Iserles, A. 2009.
\newblock \emph{A first course in the numerical analysis of differential equations}.
\newblock 44. Cambridge university press.

\bibitem[{Kingma and Ba(2014)}]{kingma2014adam}
Kingma, D.~P.; and Ba, J. 2014.
\newblock Adam: A method for stochastic optimization.
\newblock \emph{arXiv preprint arXiv:1412.6980}.

\bibitem[{Krishnapriyan et~al.(2021)Krishnapriyan, Gholami, Zhe, Kirby, and Mahoney}]{krishnapriyan2021characterizing}
Krishnapriyan, A.; Gholami, A.; Zhe, S.; Kirby, R.; and Mahoney, M.~W. 2021.
\newblock Characterizing possible failure modes in physics-informed neural networks.
\newblock \emph{Advances in Neural Information Processing Systems}, 34: 26548--26560.

\bibitem[{Lagaris, Likas, and Fotiadis(1998)}]{lagaris1998artificial}
Lagaris, I.~E.; Likas, A.; and Fotiadis, D.~I. 1998.
\newblock Artificial neural networks for solving ordinary and partial differential equations.
\newblock \emph{IEEE transactions on neural networks}, 9(5): 987--1000.

\bibitem[{LeVeque(2007)}]{leveque2007finite}
LeVeque, R.~J. 2007.
\newblock \emph{Finite difference methods for ordinary and partial differential equations: steady-state and time-dependent problems}.
\newblock SIAM.

\bibitem[{Li et~al.(2024)Li, Chen, Shan, and Huang}]{li2024causality}
Li, Y.; Chen, S.; Shan, B.; and Huang, S.-J. 2024.
\newblock Causality-enhanced discreted physics-informed neural networks for predicting evolutionary equations.
\newblock In \emph{Proceedings of the Thirty-Third International Joint Conference on Artificial Intelligence}, 4497--4505.

\bibitem[{Liu et~al.(2025)Liu, Wang, Vaidya, Ruehle, Halverson, Soljacic, Hou, and Tegmark}]{liu2025kan}
Liu, Z.; Wang, Y.; Vaidya, S.; Ruehle, F.; Halverson, J.; Soljacic, M.; Hou, T.~Y.; and Tegmark, M. 2025.
\newblock {KAN}: Kolmogorov{\textendash}Arnold Networks.
\newblock In \emph{The Thirteenth International Conference on Learning Representations}.

\bibitem[{Nguyen et~al.(2023)Nguyen, Brandstetter, Kapoor, Gupta, and Grover}]{nguyen2023climax}
Nguyen, T.; Brandstetter, J.; Kapoor, A.; Gupta, J.~K.; and Grover, A. 2023.
\newblock Climax: A foundation model for weather and climate.
\newblock \emph{arXiv preprint arXiv:2301.10343}.

\bibitem[{Pathak et~al.(2022)Pathak, Subramanian, Harrington, Raja, Chattopadhyay, Mardani, Kurth, Hall, Li, Azizzadenesheli et~al.}]{pathak2022fourcastnet}
Pathak, J.; Subramanian, S.; Harrington, P.; Raja, S.; Chattopadhyay, A.; Mardani, M.; Kurth, T.; Hall, D.; Li, Z.; Azizzadenesheli, K.; et~al. 2022.
\newblock Fourcastnet: A global data-driven high-resolution weather model using adaptive fourier neural operators.
\newblock \emph{arXiv preprint arXiv:2202.11214}.

\bibitem[{Pazy(2012)}]{pazy2012semigroups}
Pazy, A. 2012.
\newblock \emph{Semigroups of linear operators and applications to partial differential equations}, volume~44.
\newblock Springer Science \& Business Media.

\bibitem[{Penwarden et~al.(2023)Penwarden, Jagtap, Zhe, Karniadakis, and Kirby}]{penwarden2023unified}
Penwarden, M.; Jagtap, A.~D.; Zhe, S.; Karniadakis, G.~E.; and Kirby, R.~M. 2023.
\newblock A unified scalable framework for causal sweeping strategies for physics-informed neural networks (PINNs) and their temporal decompositions.
\newblock \emph{Journal of Computational Physics}, 493: 112464.

\bibitem[{Raissi, Perdikaris, and Karniadakis(2019)}]{raissi2019physics}
Raissi, M.; Perdikaris, P.; and Karniadakis, G.~E. 2019.
\newblock Physics-informed neural networks: A deep learning framework for solving forward and inverse problems involving nonlinear partial differential equations.
\newblock \emph{Journal of Computational physics}, 378: 686--707.

\bibitem[{Raissi, Yazdani, and Karniadakis(2020)}]{raissi2020hidden}
Raissi, M.; Yazdani, A.; and Karniadakis, G.~E. 2020.
\newblock Hidden fluid mechanics: Learning velocity and pressure fields from flow visualizations.
\newblock \emph{Science}, 367(6481): 1026--1030.

\bibitem[{Rasp et~al.(2020)Rasp, Dueben, Scher, Weyn, Mouatadid, and Thuerey}]{rasp2020weatherbench}
Rasp, S.; Dueben, P.~D.; Scher, S.; Weyn, J.~A.; Mouatadid, S.; and Thuerey, N. 2020.
\newblock WeatherBench: a benchmark data set for data-driven weather forecasting.
\newblock \emph{Journal of Advances in Modeling Earth Systems}, 12(11): e2020MS002203.

\bibitem[{Schrittwieser et~al.(2020)Schrittwieser, Antonoglou, Hubert, Simonyan, Sifre, Schmitt, Guez, Lockhart, Hassabis, Graepel et~al.}]{schrittwieser2020mastering}
Schrittwieser, J.; Antonoglou, I.; Hubert, T.; Simonyan, K.; Sifre, L.; Schmitt, S.; Guez, A.; Lockhart, E.; Hassabis, D.; Graepel, T.; et~al. 2020.
\newblock Mastering atari, go, chess and shogi by planning with a learned model.
\newblock \emph{Nature}, 588(7839): 604--609.

\bibitem[{Thomas(2013)}]{thomas2013numerical}
Thomas, J.~W. 2013.
\newblock \emph{Numerical partial differential equations: finite difference methods}, volume~22.
\newblock Springer Science \& Business Media.

\bibitem[{Verma, Heinonen, and Garg(2024)}]{verma2024climode}
Verma, Y.; Heinonen, M.; and Garg, V. 2024.
\newblock Clim{ODE}: Climate Forecasting With Physics-informed Neural {ODE}s.
\newblock In \emph{The Twelfth International Conference on Learning Representations}.

\bibitem[{Wang, Sankaran, and Perdikaris(2024)}]{wang2024respecting}
Wang, S.; Sankaran, S.; and Perdikaris, P. 2024.
\newblock Respecting causality for training physics-informed neural networks.
\newblock \emph{Computer Methods in Applied Mechanics and Engineering}, 421: 116813.

\bibitem[{Wight and Zhao(2020)}]{wight2020solving}
Wight, C.~L.; and Zhao, J. 2020.
\newblock Solving Allen-Cahn and Cahn-Hilliard equations using the adaptive physics informed neural networks.
\newblock \emph{arXiv preprint arXiv:2007.04542}.

\bibitem[{Wong et~al.(2022)Wong, Ooi, Gupta, and Ong}]{wong2022learning}
Wong, J.~C.; Ooi, C.~C.; Gupta, A.; and Ong, Y.-S. 2022.
\newblock Learning in sinusoidal spaces with physics-informed neural networks.
\newblock \emph{IEEE Transactions on Artificial Intelligence}, 5(3): 985--1000.

\bibitem[{Wu et~al.(2024)Wu, Luo, Ma, Wang, and Long}]{wu2024ropinn}
Wu, H.; Luo, H.; Ma, Y.; Wang, J.; and Long, M. 2024.
\newblock RoPINN: Region Optimized Physics-Informed Neural Networks.
\newblock In \emph{Advances in Neural Information Processing Systems}.

\bibitem[{Xu et~al.(2025)Xu, Liu, Hu, Li, Qin, Zheng, and Xiong}]{xu2025sub}
Xu, C.; Liu, D.; Hu, Y.; Li, J.; Qin, R.; Zheng, Q.; and Xiong, J. 2025.
\newblock Sub-Sequential Physics-Informed Learning with State Space Model.
\newblock \emph{arXiv preprint arXiv:2502.00318}.

\bibitem[{Zhao, Ding, and Prakash(2024)}]{zhao2024pinnsformer}
Zhao, Z.; Ding, X.; and Prakash, B.~A. 2024.
\newblock {PINN}sFormer: A transformer-based framework For physics-informed neural networks.
\newblock In \emph{The Twelfth International Conference on Learning Representations}.

\end{thebibliography}
\clearpage
\appendix
\setcounter{secnumdepth}{2} 
\renewcommand{\thesection}{\Alph{section}}
\renewcommand{\thesubsection}{\thesection.\arabic{subsection}}
\section{Proofs}
\label{app:proof}
\begin{theorem}[Error Propagation in PINNs]
\label{thm:error_prop_app}
For non-autoregressive PINNs, the error at step \(t_{n+1}\) is bounded by the sum of the propagated error from the previous step \(t_n\) and the one-step rollout error \(\delta_n\):
\begin{equation}
    \|e_{n+1}\|_2 \le L_{\mathcal{G}} \cdot \|e_n\|_2 + \delta_n,
\end{equation}
where \(L_{\mathcal{G}}\) is the Lipschitz constant of the true evolution operator \(\mathcal{G}(\Delta t)\).
\end{theorem}

\begin{proof}
The proof relies on the key assumption that the true evolution operator, $\mathcal{G}(\Delta t)$, is Lipschitz continuous with a constant $L_{\mathcal{G}}$. This property ensures that the operator does not excessively amplify differences between input states, and it is formally stated as:
\begin{equation}
    \| \mathcal{G}(\Delta t)[a] - \mathcal{G}(\Delta t)[b] \|_2 \le L_{\mathcal{G}} \cdot \|a - b\|_2
\end{equation}
for any two system states $a$ and $b$.

We begin with the definition of the error at time step $t_{n+1}$:
\begin{equation}
    e_{n+1} = u_\theta(x, t_{n+1}) - u_{true}(x, t_{n+1}).
\end{equation}
By definition, the true solution at $t_{n+1}$ is given by the evolution operator $\mathcal{G}(\Delta t)$ applied to the true solution at $t_n$. Substituting this into our error expression gives:
\begin{equation}
    e_{n+1} = u_\theta(x, t_{n+1}) - \mathcal{G}(\Delta t) [u_{true}(x, t_n)].
\end{equation}
We now add and subtract the term $\mathcal{G}(\Delta t) [u_\theta(x, t_n)]$. This allows us to connect the model's prediction at $t_{n+1}$ to the evolution of its own prediction from $t_n$:
\begin{align}
    e_{n+1} = & \left( u_\theta(x, t_{n+1}) - \mathcal{G}(\Delta t) [u_\theta(x, t_n)] \right) \notag \\
    & + \left( \mathcal{G}(\Delta t) [u_\theta(x, t_n)] - \mathcal{G}(\Delta t) [u_{true}(x, t_n)] \right).
\end{align}
Taking the $L_2$ norm and applying the triangle inequality ($\|A+B\| \le \|A\| + \|B\|$) yields:
\begin{align}
    \|e_{n+1}\|_2 \le & \left\| u_\theta(x, t_{n+1}) - \mathcal{G}(\Delta t) [u_\theta(x, t_n)] \right\|_2 \notag \\
    & + \left\| \mathcal{G}(\Delta t) [u_\theta(x, t_n)] - \mathcal{G}(\Delta t) [u_{true}(x, t_n)] \right\|_2. \label{eq:triangle}
\end{align}
We recognize the first term on the right-hand side as the definition of the one-step rollout error, $\delta_n$. For the second term, we apply the Lipschitz continuity of the operator $\mathcal{G}$:
\begin{align}
    &\left\| \mathcal{G}(\Delta t) [u_\theta(x, t_n)] - \mathcal{G}(\Delta t) [u_{true}(x, t_n)] \right\|_2 \notag \\
    &\qquad \le L_{\mathcal{G}} \cdot \| u_\theta(x, t_n) - u_{true}(x, t_n) \|_2 \notag \\
    &\qquad = L_{\mathcal{G}} \cdot \|e_n\|_2.
\end{align}
Substituting these two results back into Equation~\eqref{eq:triangle}, we arrive at the final inequality:
\begin{equation}
    \|e_{n+1}\|_2 \le \delta_n + L_{\mathcal{G}} \cdot \|e_n\|_2.
\end{equation}
Rearranging the terms gives the statement of the theorem:
\begin{equation}
    \|e_{n+1}\|_2 \le L_{\mathcal{G}} \cdot \|e_n\|_2 + \delta_n.
\end{equation}
\end{proof}

\section{Existence, Uniqueness and Regularity of Solutions for Evolution Operators}
We study time-dependent PDEs of the form
\begin{equation} \label{eq:pde}
\frac{du}{dt} = -A u + f(u), \quad u(0) = u_0 \in H,
\end{equation}
where $H$ is a complete normed vector space with an inner product (Hilbert space), $-A: D(A) \subset H \to H$ is a linear operator generating a strongly continuous semigroup, and $f: H \to H$ is a nonlinear operator.

\subsection{Semigroup Setup and Mild Solutions}

\begin{definition}[Semigroup Generator]
Let $-A: D(A) \subset H \to H$ be a closed, densely defined linear operator such that $-A$ generates a strongly continuous solution operator $\{S(t)\}_{t \geq 0} \subset \mathcal{L}(H)$; and $\{S(t)\}_{t \geq 0}$ satisfies :
\begin{enumerate}
    \item $S(t)u_0(\cdot) = u(\cdot, t)$
    \item $S(t)$ is a semigroup which satisifies $S(0) = I$ (the identity) and $S(t+s) = S(t)S(s)$ for all $t, s \geq 0$;
    \item $\|S(t)\| \leq 1$ for all $t \geq 0$.
\end{enumerate}
Starting at time $t=0$, we have $u(\cdot, 0) = S(0)u_0(\cdot)$. You obtain $u(\cdot, s+t)$ by first flowing forward in time by $s$ and then flow forward by time $t$ using $u(\cdot,s)$ as initial data.
\end{definition}

\begin{definition}[Mild Solution]
A function $u \in C([0,T], H)$ is a \emph{mild solution} to \eqref{eq:pde} if it satisfies the variation of constants formula:
\begin{equation}
u(\cdot, t) = S(t)u_0(\cdot) + \int_0^t S(t - s) f(u(\cdot, s)) \, ds.
\end{equation}
\end{definition}

\begin{theorem}[Existence and Uniqueness of Mild Solutions]
Suppose $f : H \to H$ is Lipschitz and satisfies:
\begin{align} \label{eqn:lipschitz_condition}
   \|f(u) - f(v)\| \leq L\|u - v\|, \\\quad \|f(u)\| \leq L(1 + \|u\|), \quad \forall u, v \in H. 
\end{align}

Then for any $u_0 \in H$, there exists a unique mild solution $u \in C([0,T], H)$, and:
\[
\|u(t)\| \leq C_T(1 + \|u_0\|), \quad \forall t \in [0,T].
\]
\end{theorem}

\begin{proof}
Let $X := C([0,T], H)$ with norm $\|u\|_X := \sup_{t \in [0,T]} \|u(t)\|$. Define the mapping $\mathcal{J} : X \to X$ by
\[
(\mathcal{J}u)(t) := S(t)u_0 + \int_0^t S(t - s) f(u(s)) \, ds.
\]
We show $\mathcal{J}$ is a contraction for small $T$:
\begin{align*}
\|\mathcal{J}u - \mathcal{J}v\|_X &\leq \sup_{t \in [0,T]} \int_0^t \|f(u(s)) - f(v(s))\| ds \\
&\leq L T \|u - v\|_X.
\end{align*}
Choose $T < 1/L$ so that $\mathcal{J}$ is a contraction. By Banach’s fixed-point theorem, there exists a unique fixed point $u \in X$. Repeating over intervals gives global existence.

For the bound, using $\|S(t)\| \leq 1$:
\[
\|u(t)\| \leq \|u_0\| + \int_0^t L(1 + \|u(s)\|) ds.
\]
Apply Grönwall’s inequality to conclude.
\end{proof}

\subsection{Regularity and Temporal Smoothness}

\begin{theorem}[Temporal Regularity]\label{thm:temporal_smoothness}
Let $u_0 \in D(A^\gamma)$ for some $\gamma \in (0,1]$, and let $u(t)$ be the mild solution. Then for any $0 \leq t_1 \leq t_2 \leq T$ and $\epsilon > 0$, there exists $C > 0$ such that:
\[
\|u(t_2) - u(t_1)\| \leq C |t_2 - t_1|^\theta (1 + \|u_0\|_{D(A^\gamma)}),
\]
where $\theta = \min\{\gamma, 1 - \epsilon\}$.
\end{theorem}

\begin{proof}[Sketch]
Write $u(t_2) - u(t_1) = I + II$, with:
\begin{align*}
    I &:= S(t_2)u_0 - S(t_1)u_0, \quad \\II &:= \int_0^{t_2} S(t_2 - s)f(u(s))ds - \int_0^{t_1} S(t_1 - s)f(u(s))ds.
\end{align*}

Estimate $I$ via semigroup regularity:
\[
\|I\| \leq C |t_2 - t_1|^\gamma \|u_0\|_{D(A^\gamma)}.
\]
For $II$, split as:
\begin{align*}
   II_1 &:= \int_0^{t_1} [S(t_2 - s) - S(t_1 - s)] f(u(s)) ds, \\
   \quad II_2 &:= \int_{t_1}^{t_2} S(t_2 - s) f(u(s)) ds, 
\end{align*}

and use continuity of $S(t)$ and boundedness of $f(u(s))$ to obtain:
\[
\|II\| \leq C |t_2 - t_1|^\theta (1 + \|u_0\|).
\]
\end{proof}

\subsection{Evolution Operator Approximation}

\begin{definition}[Evolution Operator $\mathcal{G}$]
Let $\Delta t > 0$. Define a learned operator $\mathcal{G} : H \to H$ such that:
\[
\mathcal{G}(u_n) \approx u_{n+1}, \quad \text{where } u_n \approx u(n\Delta t), \ u_{n+1} \approx u((n+1)\Delta t).
\]
Let $\Phi_{\Delta t}(u)$ denote the exact flow:
\[
\Phi_{\Delta t}(u) := S(\Delta t)u + \int_0^{\Delta t} S(\Delta t - s) f(u(s)) ds.
\]
\end{definition}

\begin{theorem}[Error Propagation of $\mathcal{G}$] \label{thm:error_propagation}
Let $u(\cdot, t)$ be the mild solution to \eqref{eq:pde} with $u_0 \in D(A^\gamma)$ for some $\gamma \in (0,1]$. Suppose:
\[
\|\mathcal{G}(u) - \Phi_{\Delta t}(u)\| \leq \varepsilon(\Delta t), \quad \forall u \in \mathcal{B} \subset H.
\]
Define $\tilde{u}_0 = u_0$, and recursively $\tilde{u}_{n+1} = \mathcal{G}(\tilde{u}_n)$. Then for $u_n := u(n \Delta t)$,
\[
\|\tilde{u}_n - u_n\| \leq C_T \left( \varepsilon(\Delta t) + \Delta t^\theta (1 + \|u_0\|_{D(A^\gamma)}) \right),
\]
where $\theta = \min\{\gamma, 1 - \epsilon\}$ and $C_T$ depends on $T$ and $L$.
\end{theorem}

\begin{proof}
We proceed by induction.

Base case: $\tilde{u}_0 = u_0 \Rightarrow \|\tilde{u}_0 - u_0\| = 0$.

Inductive step:
\begin{align*}
\|\tilde{u}_{n+1} - u_{n+1}\| &\leq \|\mathcal{G}(\tilde{u}_n) - \Phi_{\Delta t}(\tilde{u}_n)\| \\& + \|\Phi_{\Delta t}(\tilde{u}_n) - \Phi_{\Delta t}(u_n)\| \\
&\leq \varepsilon(\Delta t) + L_\Phi \|\tilde{u}_n - u_n\|.
\end{align*}
Apply recursively and use the regularity bound from \ref{thm:temporal_smoothness} for $\|u_{n+1} - \Phi_{\Delta t}(u_n)\|$ to close the estimate.
\end{proof}
While PIANO mitigates the recurrence  error by minimizing the one-step rollout term, Theorem \ref{thm:error_propagation}  proves that autoregressive models that approximate the true evolution operator enjoy  a provably bounded global error.

\section{Training Algorithm}
\label{app:training}

\begin{algorithm}[h]
\caption{Training PIANO via Experience Learning}
\label{alg:piano-training}
\begin{algorithmic}[1]
\STATE Initialize model parameters \( \psi, \phi, \eta \).
\FOR{each training iteration}
    \STATE Sample a batch of spatial coordinates \( \{x_i\}_{i=1}^{N_x} \subset \Omega \cup \partial\Omega \).
    \STATE Set initial state from the known condition: \( \hat{u}(x_i, t_0) \leftarrow u(x_i, t_0) \) for all \( i \).
    \STATE Initialize hidden state \( h^i_0 \leftarrow \mathbf{0} \) for all \( i \).
    \FOR{each time step \( t_j \), for \( j = 1, \dots, M \)}
        \STATE Form input vector: \( s^i_j = (x_i, t_j, \hat{u}(x_i, t_{j-1})) \).
        \STATE Compute embedding: \( m^i_j = p_\psi(s^i_j) \).
        \STATE Update hidden state and output representation: \( (h^i_j, o^i_j) = f_\phi(h^i_{j-1}, m^i_j) \).
        \STATE Predict solution: \( \hat{u}(x_i, t_j) = q_\eta(o^i_j) \).
        \STATE Compute loss using Eq. \ref{eq:piano_loss_energy}
    \ENDFOR
    \STATE Normalize loss over time steps and batch size.
    \STATE Update parameters \( \psi, \phi, \eta \).
\ENDFOR
\end{algorithmic}
\end{algorithm}

This training procedure, which we refer to as Physics-Informed Experience Learning (PIEL), optimizes the model to generate physically consistent solution trajectories based on its own predictions. The experience learning component comes from the autoregressive rollout described in Algorithm~\ref{alg:piano-training}: for each spatial coordinate \( x_i \) in a batch, the model generates an entire temporal trajectory starting from the known initial condition \( u(x_i, t_0) \). Each subsequent prediction \( \hat{u}(x_i, t_j) \) is conditioned on the model's own previous output \( \hat{u}(x_i, t_{j-1}) \), which forces the model to learn from its own generated experience.

The physics-informed component governs the optimization process. Instead of comparing the predicted rollout to a ground-truth solution, the loss function measures how well the generated trajectory satisfies the governing physical laws. This is done by evaluating the residuals of the underlying partial differential equation (PDE), as well as the errors in satisfying the boundary conditions. These residuals are computed over the full predicted spatiotemporal grid, using finite difference approximations for both spatial and temporal derivatives. The total loss is then aggregated over all points and time steps in the trajectory.

The model parameters are updated through backpropagation through time (BPTT). By maintaining gradient flow through the full autoregressive sequence, the model learns a stable state transition function, or evolution operator, that captures long-range temporal dependencies and adheres to the physical constraints. This end-to-end training on physically constrained rollouts directly minimizes the one-step rollout error discussed in Theorem~3.4. As a result, the model mitigates error accumulation commonly found in non-autoregressive approaches and produces stable, accurate long-term predictions.

\section{Additional Details on PDE Benchmark}
\label{app:implementation_details}

\subsection{PDE Setup}
\label{app:pde_setup}

We evaluate PIANO on four canonical time-dependent partial differential equations (PDEs) that are standard benchmarks in the physics-informed machine learning literature. For each benchmark, we define the governing equation, the domain, the initial conditions (ICs), and the boundary conditions (BCs). The analytical solution for each PDE is provided for the purpose of evaluating model accuracy. Table~\ref{tab:pde_challenges} summarizes the unique numerical challenges posed by each equation, which test different aspects of a solver's stability and accuracy.

\noindent\textbf{1. Wave Equation:}
The 1D wave equation models phenomena like vibrating strings and sound waves. It is a second-order hyperbolic PDE.
\begin{itemize}
    \item \textbf{Equation}: $\frac{\partial^2 u}{\partial t^2} = c^2 \frac{\partial^2 u}{\partial x^2}$, with $c=2.0$.
    \item \textbf{Domain}: $(x, t) \in [0, 1] \times [0, 1]$.
    \item \textbf{Initial Conditions}:
    \begin{itemize}
        \item $u(x, 0) = \sin(\pi x) + 0.5 \sin(3\pi x)$.
        \item $\frac{\partial u}{\partial t}(x, 0) = 0$.
    \end{itemize}
    \item \textbf{Boundary Conditions}: $u(0, t) = 0$ and $u(1, t) = 0$ (Dirichlet).
    \item \textbf{Analytical Solution}: $u(x, t) = \sin(\pi x)\cos(2\pi t) + 0.5\sin(3\pi x)\cos(6\pi t)$.
\end{itemize}

\noindent\textbf{2. Reaction Equation:}
This equation models systems with nonlinear reaction dynamics, common in chemistry and biology. It is a first-order nonlinear PDE.
\begin{itemize}
    \item \textbf{Equation}: $\frac{\partial u}{\partial t} = 5u(1-u)$.
    \item \textbf{Domain}: $(x, t) \in [0, 2\pi] \times [0, 1]$.
    \item \textbf{Initial Condition}: $u(x, 0) = \exp\left(-\frac{(x-\pi)^2}{2(\pi/4)^2}\right)$.
    \item \textbf{Boundary Conditions}: $u(0, t) = u(2\pi, t)$ (Periodic).
    \item \textbf{Analytical Solution}: Let $h(x) = u(x,0)$. Then $u(x,t) = \frac{h(x)e^{5t}}{h(x)e^{5t} + 1 - h(x)}$.
\end{itemize}

\noindent\textbf{3. Convection Equation:}
A first-order hyperbolic PDE that models the transport of a quantity. It is known for being sensitive to numerical diffusion, where sharp features can be smoothed out by inaccurate solvers.
\begin{itemize}
    \item \textbf{Equation}: $\frac{\partial u}{\partial t} + c \frac{\partial u}{\partial x} = 0$, with $c=50$.
    \item \textbf{Domain}: $(x, t) \in [0, 2\pi] \times [0, 1]$.
    \item \textbf{Initial Condition}: $u(x, 0) = \sin(x)$.
    \item \textbf{Boundary Conditions}: $u(0, t) = u(2\pi, t)$ (Periodic).
    \item \textbf{Analytical Solution}: $u(x, t) = \sin(x - ct)$.
\end{itemize}

\noindent\textbf{4. Heat Equation:}
The heat equation is a second-order parabolic PDE that describes heat distribution in a given region over time. It is a classic example of a diffusive system, which presents challenges related to numerical stiffness.
\begin{itemize}
    \item \textbf{Equation}: $\frac{\partial u}{\partial t} = \alpha \frac{\partial^2 u}{\partial x^2}$, with $\alpha=0.1$.
    \item \textbf{Domain}: $(x, t) \in [0, 1] \times [0, 1]$.
    \item \textbf{Initial Condition}: $u(x, 0) = \sin(\pi x)$.
    \item \textbf{Boundary Conditions}: $u(0, t) = 0$ and $u(1, t) = 0$ (Dirichlet).
    \item \textbf{Analytical Solution}: $u(x, t) = \sin(\pi x)e^{-\alpha\pi^2 t}$.
\end{itemize}

\begin{table*}[h]
\centering
\begin{tabular}{@{}ll@{}}
\toprule
\textbf{PDE Benchmark} & \textbf{Primary Numerical Challenge} \\
\midrule
\textbf{Wave Equation} & \makecell[l]{Accurate handling of second-order temporal and spatial derivatives. \\ Propagating wave solutions without numerical dispersion.} \\[1em]
\textbf{Reaction Equation} & \makecell[l]{Modeling stiff nonlinear dynamics where solutions change rapidly. \\ Capturing exponential growth accurately.} \\[1em]
\textbf{Convection Equation} & \makecell[l]{Minimizing numerical diffusion to preserve the shape of the \\ propagating wave. High wave speed ($c=50$) makes it challenging.} \\[1em]
\textbf{Heat Equation} & \makecell[l]{Handling diffusive processes and numerical stiffness, which can \\ require very small time steps for traditional explicit solvers.} \\
\bottomrule
\end{tabular}
\caption{Summary of numerical challenges presented by each PDE benchmark.}
\label{tab:pde_challenges}
\end{table*}

\subsection{Baselines}
\label{app:baselines}

To rigorously evaluate \textbf{PIANO}, we benchmark it against a comprehensive suite of models that represent the cutting edge of physics-informed learning. These baselines were chosen to cover classical architectures, recent architectural innovations, and state-of-the-art sequential models, providing a multi-faceted comparison.

\subsubsection{Baselines}
\begin{itemize}
    \item \textbf{Canonical PINN}: This is the foundational framework introduced by \citet{raissi2019physics}. It employs a Multilayer Perceptron (MLP) as a universal function approximator that takes spatio-temporal coordinates $(x,t)$ as input and outputs the corresponding solution $u(x,t)$. The network is trained by minimizing a loss function composed of the PDE residuals, initial conditions, and boundary conditions, which are calculated using automatic differentiation. Its primary limitation, which motivates our work, is its point-wise prediction mechanism, which neglects temporal dependencies and often fails to propagate initial conditions accurately, leading to "failure modes" where the model converges to overly smooth or incorrect solutions.

    \item \textbf{QRes \& FLS}: These models represent architectural improvements over the standard MLP. \textbf{QRes (Quadratic Residual Networks)} introduces quadratic residual connections to enhance the model's capacity for solving complex physics problems \citep{bu2021quadratic}. \textbf{FLS (First-Layer Sine)} networks use a sinusoidal activation function in the first layer \citep{wong2022learning}. This provides a strong inductive bias for learning periodic or high-frequency patterns, though its effectiveness can be limited to problems where such prior knowledge of the solution's behavior is applicable.

    \item \textbf{RoPINN (Region Optimized PINN)}: This framework addresses a fundamental deficiency in the standard PINN training paradigm \citep{wu2024ropinn}. Instead of optimizing the loss on a finite set of scattered points, RoPINN extends the optimization to the continuous neighborhood regions around these points. This is achieved efficiently through a Monte Carlo sampling method within a "trust region" that is adaptively calibrated during training. This approach is designed to reduce generalization error and better satisfy high-order PDE constraints without requiring additional, costly gradient calculations.

    \item \textbf{KAN (Kolmogorov-Arnold Networks)}: Representing a recent breakthrough in neural network architecture, KANs are included as an advanced physics-informed backbone \citep{liu2025kan}. They offer a powerful alternative to traditional MLPs and have demonstrated strong performance in several scientific machine learning tasks.

    \item \textbf{Sequential, Non-Autoregressive Models}: To highlight the specific advantage of PIANO's autoregressive nature, we compare against the most advanced sequential models.
    \begin{itemize}
        \item \textbf{PINNsFormer} is a Transformer-based framework designed specifically to capture temporal dependencies \citep{zhao2024pinnsformer}. Its key mechanism is the "Pseudo Sequence Generator," which transforms a point-wise input $(x,t)$ into a short temporal sequence, $\{[x, t], [x, t+\Delta t], ...\}$, which is then processed by a multi-head attention mechanism. While it processes information sequentially, its predictions at each time step are made independently of the model's own previous predictions, distinguishing it from PIANO's true autoregressive approach.
        \item \textbf{PINNMamba} utilizes a State Space Model (SSM) to serve as a ``continuous-discrete articulation," aiming to resolve the mismatch between continuous PDEs and discrete training points \citep{xu2025sub}. It employs ``sub-sequence modeling" and a contrastive alignment loss to combat the simplicity bias of neural networks and propagate initial conditions. Like PINNsFormer, it is a sequential model, but it is not autoregressive in the way PIANO is.
    \end{itemize}
\end{itemize}

\subsection{Hyperparameters and Experimental Details}
\label{app:hyperparameters}
To ensure full reproducibility of our results, this section provides comprehensive information on the evaluation metrices, experimental setup, training configurations, and specific hyperparameters used for both our proposed model, \textbf{PIANO}, and all baseline models.

\subsubsection{Experimental Setup}
All experiments were conducted on a workstation equipped with NVIDIA A100 GPUs with GPU memory usage of $\sim$800MiB. Our codebase is implemented in PyTorch. For the PDE benchmark tasks, we trained all models on a discretized spatio-temporal grid of $200 \times 200$ collocation points. For evaluation and visualization, a distinct grid of $198 \times 198$ points was used. Although PIANO does not involve significant sources of randomness that affect performance, experiments are repeated with seeds 0 through 9 for multiple-run evaluations. For single-run experiments, seed 0 is used. We provide the source code in the attached zip file with software libraries and their versions in the .toml file. The source code will also be made public upon acceptance.

\subsubsection{Training Configuration}
PIANO was trained for 100K iterations using the AdamW optimizer with a learning rate of $3 \times 10^{-4}$ and a weight decay of $10^{-4}$. A cosine annealing learning rate scheduler was employed to gradually reduce the learning rate over the course of training. To stabilize optimization, gradient norms were clipped to a maximum of 1.0. Model weights were initialized using the Xavier uniform initialization strategy, and the best-performing model (based on training loss) was checkpointed and used for final evaluation. Throughout training, intermediate predictions were saved at regular intervals to support qualitative analysis. All spatial and temporal derivatives required for the PDE loss were approximated using second-order finite difference schemes applied to the model’s predicted solution grid. When enabled, Weights \& Biases (WandB) was used to log training losses, runtime, and gradient diagnostics. All models were trained on a uniform spatiotemporal grid, with analytically defined initial conditions specific to each PDE benchmark.

\subsubsection{Model Architecture Details}
Our proposed model, PIANO, is implemented as a State-Space Model (SSM) architecture. The key components and hyperparameters are detailed below, followed by those of the primary baselines.

\begin{itemize}
    \item \textbf{PIANO (Ours)}: Our model, referred to as `ssm` in our experiments, is an autoregressive state-space model.
    \begin{itemize}
        \item \textbf{Input}: At each time step $t$, the input vector is a concatenation of the spatial coordinate $x_t$, the temporal coordinate $t_t$, and the model's own prediction from the previous step, $u_{t-1}$.
        \item \textbf{Embedding Network}: A linear layer maps the concatenated input vector to a hidden state dimension of 256.
        \item \textbf{State Transition Network}: The core of our model consists of learnable matrices (\textbf{A, B, C, D}) that govern the state-space dynamics. The hidden state dimension is 256, and we use a SiLU (silu) activation function with Layer Normalization for stability.
        \item \textbf{PDE Probe}: A 2-layer MLP with a hidden dimension of 256 and a SiLU activation decodes the hidden state into the final solution at each time step.
    \end{itemize}
\end{itemize}

Table \ref{tab:hyperparams} summarizes the architectural hyperparameters for PIANO and the primary baseline models to ensure a fair comparison in terms of model capacity. We aimed to keep the total number of trainable parameters roughly comparable across the different architectures.

\begin{table}[t]
\centering
\resizebox{\columnwidth}{!}{%
\begin{tabular}{@{}llc@{}}
\toprule
\textbf{Model} & \textbf{Hyperparameter} & \textbf{Value} \\
\midrule
\multirow{3}{*}{\textbf{PINN / FLS}} & Hidden Layers & 4 \\
& Hidden Size & 512 \\
& Total Parameters & $\sim$527k \\
& GPU Memory & $\sim$1311MiB \\
\midrule
\multirow{3}{*}{\textbf{QRes}} & Hidden Layers & 4 \\
& Hidden Size & 256 \\
& Total Parameters & $\sim$397k \\
\midrule
\multirow{7}{*}{\textbf{PINNsFormer}} & Sequence Length ($k$) & 5 \\
& Time Step ($\Delta t$) & $10^{-4}$ \\
& Number of Encoders/Decoders & 1 \\
& Embedding Size & 32 \\
& Attention Heads & 2 \\
& Hidden Size & 512 \\
& Total Parameters & $\sim$454k \\
& GPU Memory & $\sim$2827MiB \\
\midrule
\multirow{5}{*}{\textbf{PINNMamba}} & Sequence Length ($k$) & 7 \\
& Time Step ($\Delta t$) & $10^{-2}$ \\
& Number of Encoders & 1 \\
& Embedding Size & 32 \\
& Total Parameters & $\sim$286k \\
& GPU Memory & $\sim$7899MiB \\
\midrule
\multirow{3}{*}{\textbf{PIANO (Ours)}} & State Dimension ($k$) & 256 \\
& Total Parameters & $\sim$330k \\
& GPU Memory & $\sim$800MiB \\
\bottomrule
\end{tabular}%
}
\caption{Architectural hyperparameters for PIANO and baseline models used in the PDE benchmarks.}
\label{tab:hyperparams}
\end{table}

\subsubsection{Guidance on hyperparameters} The architectural hyperparameters in PIANO, namely the state dimension ($k$) and the number of temporal rollout steps ($M$), are standard and do not introduce novel tuning complexities.  The value for $M$ is determined by the temporal resolution of the training grid (e.g., for a $200\times200$ spatio-temporal grid, $M=200$). To provide clear guidance for practitioners, we perform an empirical sensitivity analysis on these key parameters. We investigate the impact of the state dimension and the training grid resolution on the model's accuracy for the Reaction equation.

The results are presented in Table~\ref{tab:sensitivity_k} and Table~\ref{tab:sensitivity_grid}. We observe two clear trends:
\begin{itemize}
    \item Increasing the state dimension from 32 to 256 substantially reduces prediction error, demonstrating the benefit of higher model capacity for capturing the underlying dynamics of the PDE.
    \item Increasing the training grid resolution from $50 \times 50$ to $200 \times 200$ consistently improves performance, as a denser sampling of collocation points provides stronger and more complete physical constraints during training.
\end{itemize}
Based on these findings, we used a state dimension of $k=256$ and a grid size of $200 \times 200$ for our main experiments to ensure the highest accuracy. We do not go further as the error is close to zero.

\begin{table}[h!]
\centering
\begin{tabular}{@{}ccc@{}}
\toprule
\textbf{State Dimension ($k$)} & \textbf{rMAE ($\downarrow$)} & \textbf{rRMSE ($\downarrow$)} \\
\midrule
32 & 0.0145 & 0.0284 \\
64 & 0.0025 & 0.0047 \\
128 & 0.0016 & 0.0038 \\
\textbf{256} & \textbf{0.0001} & \textbf{0.0008} \\
\bottomrule
\end{tabular}
\caption{Sensitivity analysis of PIANO with respect to the state dimension ($k$) for the Reaction equation. The model is trained on a $200 \times 200$ grid. Errors decrease as the state dimension increases.}
\label{tab:sensitivity_k}
\end{table}

\begin{table}[h!]
\centering
\begin{tabular}{@{}ccc@{}}
\toprule
\textbf{Grid Size} & \textbf{rMAE ($\downarrow$)} & \textbf{rRMSE ($\downarrow$)} \\
\midrule
$50 \times 50$ & 0.0017 & 0.0047 \\
$100 \times 100$ & 0.0005 & 0.0009 \\
\textbf{$200 \times 200$} & \textbf{0.0001} & \textbf{0.0008} \\
\bottomrule
\end{tabular}
\caption{Sensitivity analysis of PIANO with respect to the training grid resolution for the Reaction equation. The model uses a fixed state dimension of $k=256$. Performance improves with a finer grid.}
\label{tab:sensitivity_grid}
\end{table}

 \subsubsection{Evaluation Metrics}
\label{app:metrics}
To quantitatively assess the accuracy of our model and the baselines, we employ two primary evaluation metrics: the relative Mean Absolute Error (rMAE) and the relative Root Mean Squared Error (rRMSE). These metrics are standard and widely adopted throughout the PINN research literature, ensuring our results are comparable with prior and future work \citep{xu2025sub, wu2024ropinn, zhao2024pinnsformer}.

We use relative error metrics instead of absolute ones (e.g., MAE or RMSE) because they provide a scale-invariant measure of performance. Absolute errors are dependent on the magnitude of the PDE's solution; a physically correct model for a high-magnitude field (like pressure) could have a large absolute error, while a poor model for a normalized field (like concentration) could have a small one. Relative errors normalize the error by the magnitude of the true solution, providing a dimensionless percentage that is directly comparable across different PDEs, scales, and physical units. This is essential for a robust and generalizable evaluation.

Given a set of $N$ test points, the model's prediction $\hat{u}(x_n, t_n)$, and the ground truth analytical solution $u(x_n, t_n)$, the metrics are formulated as follows:

\begin{equation}
    \text{rMAE} = \frac{\sum_{n=1}^{N} | \hat{u}(x_n, t_n) - u(x_n, t_n) |}{\sum_{n=1}^{N} | u(x_n, t_n) |}
    \label{eq:rmae}
\end{equation}

\begin{equation}
    \text{rRMSE} = \sqrt{\frac{\sum_{n=1}^{N} | \hat{u}(x_n, t_n) - u(x_n, t_n) |^2}{\sum_{n=1}^{N} | u(x_n, t_n) |^2}}
    \label{eq:rrmse}
\end{equation}

\subsubsection{Complexity Analysis}
We provide a complexity analysis to position PIANO relative to its baselines in terms of computational and memory overhead, with model details summarized in Table~\ref{tab:hyperparams}. Standard MLP-based PINNs serve as an efficient baseline, requiring approximately 1311\,MiB of GPU memory. In contrast, advanced sequential models designed to capture temporal dependencies incur significantly greater costs. PINNsFormer, based on a Transformer architecture, exhibits quadratic computational and memory complexity with respect to sequence length $M$, i.e., $\mathcal{O}(M^2)$, due to its self-attention mechanism. This makes it inherently inefficient for long sequences and results in a memory footprint of around 2827\,MiB, more than twice that of a standard PINN, along with nearly three times the computational cost. Similarly, PINNMamba adopts a state-space architecture that theoretically scales linearly as $\mathcal{O}(M)$, but in practice incurs a very high memory footprint of approximately 7899\,MiB and nearly seven times the time per iteration of a standard PINN. This overhead stems from its reliance on short sub-sequences (e.g., $k=7$) and a complex sub-sequence contrastive alignment loss. PIANO, while also a sequential model, is explicitly designed for efficient long-range modeling. Like PINNMamba, it scales linearly with $M$ in theory. However, it avoids the need for sub-sequence processing and specialized losses. Instead, it processes the entire temporal rollout as a single long sequence (e.g., $M=200$ in our experiments) using a streamlined autoregressive architecture. As shown in Table~\ref{tab:hyperparams}, this results in a substantially lower memory usage of approximately 800\,MiB, which is lower than both PINNMamba and even the MLP-based PINN baseline. While PIANO's autoregressive nature introduces a modest increase in computational cost per iteration compared to MLPs, it remains comparable to other sequential baselines without incurring their prohibitive memory overhead. Overall, PIANO achieves an effective balance between expressive power, scalability, and resource efficiency, making it well suited for simulating long-horizon dynamical systems.

\subsection{Additional Qualitative Results}
\label{app:qualitative}

\subsubsection{Training Dynamics}
To provide further insight into the autoregressive behavior of PIANO, we visualize its training dynamics across different PDEs. Figures~\ref{fig:appendix-wave}--\ref{fig:appendix-reaction-full} show spacetime predictions \( u(x, t) \) at various stages of training (5\%, 25\%, 50\%, and 100\%), illustrating how the solution progressively improves. These visualizations highlight how the model gradually reconstructs the full trajectory by propagating the known initial condition forward in time using its own predictions.

For the transport-dominated problems like the Convection equation (main paper Figure~\ref{fig:training_dynamics}) highlight PIANO’s ability to preserve sharp wavefronts over long horizons. This demonstrates the model's robustness in avoiding numerical diffusion, an issue that often affects other PINN-based methods.

The Wave equation (Figure~\ref{fig:appendix-wave}) requires learning second-order oscillatory dynamics, making convergence more gradual. Early in training, the model underfits both amplitude and phase. However, by the end of training, PIANO successfully recovers the full oscillatory structure, without suffering from phase drift or artificial dispersion. The solution is slowly and correctly propagated from the initial stages to the later.

For the Heat equation (Figure~\ref{fig:appendix-heat}), PIANO converges extremely rapidly. As a diffusion-dominated PDE, the solution is smooth and stable, allowing the model to reach near-zero error by 50\% of training.

The Reaction equation, shown in Figures~\ref{fig:appendix-reaction} and~\ref{fig:appendix-reaction-full}, provides a particularly informative case for analyzing autoregressive propagation. Although convergence is still relatively fast, we include an extended grid of training snapshots to illustrate how the model handles the nonlinearity and exponential growth in the solution. Early predictions accurately reconstruct the region near \( t = 0 \), where the initial condition serves as an anchor. As training progresses, the predicted solution propagates deeper into the temporal domain, revealing how PIANO gradually builds up the full dynamics through stable, recursive conditioning. This dense visualization helps expose the internal mechanics of the autoregressive learning process.

Overall, these qualitative results confirm that PIANO maintains stable and physically consistent rollout behavior across a wide range of PDE types, from diffusive to oscillatory, and from linear to nonlinear dynamics.

\begin{figure*}[ht]
    \centering
    \includegraphics[width=\linewidth]{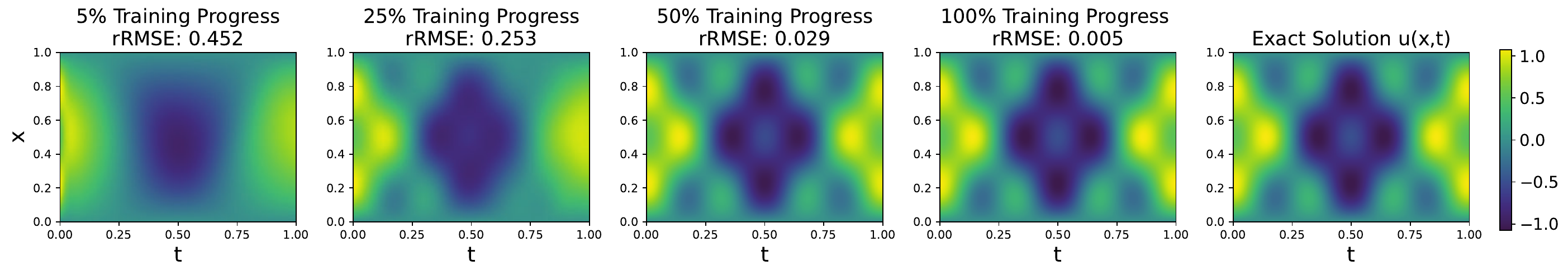}
    \caption{Training progression of PIANO on the wave equation. Each subplot shows the predicted spacetime solution \( u(x, t) \) at different training stages, with the exact solution on the far right. The model progressively learns the oscillatory structure of the solution, despite the complexity of second-order dynamics. Accurate predictions emerge early near the initial condition and gradually propagate forward in time through the autoregressive rollout. By 100\% progress (rRMSE: 0.005), the model closely matches the exact solution across the domain.}
    \label{fig:appendix-wave}
\end{figure*}

\begin{figure*}[ht]
    \centering
    \includegraphics[width=\linewidth]{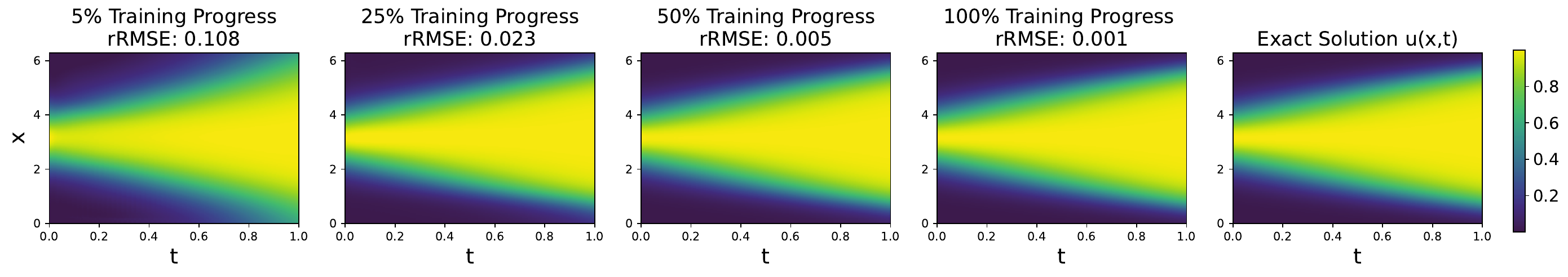}
    \caption{Training progression of PIANO on a reaction equation. Each subplot shows the predicted spacetime solution \( u(x, t) \) at different training stages, with the exact solution on the far right. At 5\% progress, the model captures the coarse global structure but underestimates the peak and shows boundary errors. As training proceeds, accurate predictions emerge first near the initial state and gradually propagate forward in time through the autoregressive rollout. By 100\% progress (rRMSE: 0.0008), the model closely matches the exact solution. Extended dynamics are shown in Figure~\ref{fig:appendix-reaction-full}.}
    \label{fig:appendix-reaction}
\end{figure*}

\begin{figure*}[ht]
    \centering
    \includegraphics[width=\linewidth]{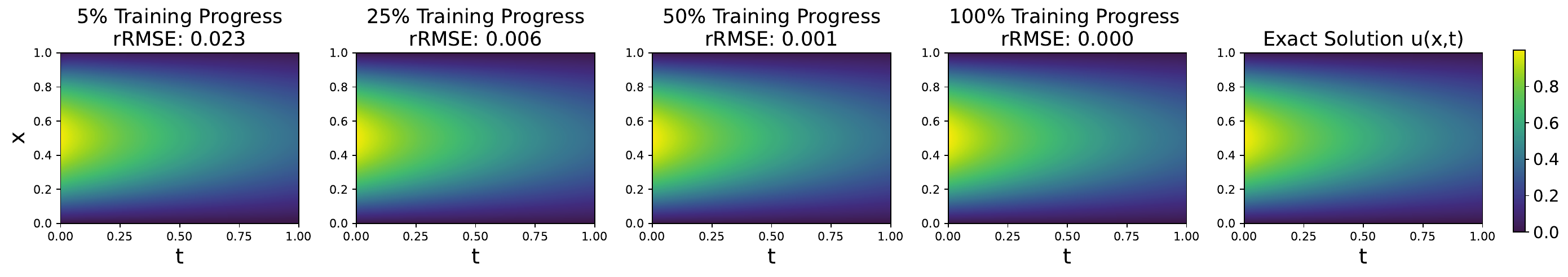}
    \caption{Training progression of PIANO on the heat equation. Each subplot shows the predicted spacetime solution \( u(x, t) \) at different training stages, with the exact solution on the far right. Due to the diffusive nature of the equation, the model converges rapidly, capturing the correct solution structure early in training. Accurate dynamics propagate smoothly from the initial condition, with near-perfect agreement reached by 50\% progress (rRMSE: 0.001).}
    \label{fig:appendix-heat}
\end{figure*}

\begin{figure*}[ht]
    \centering
    \includegraphics[width=\linewidth]{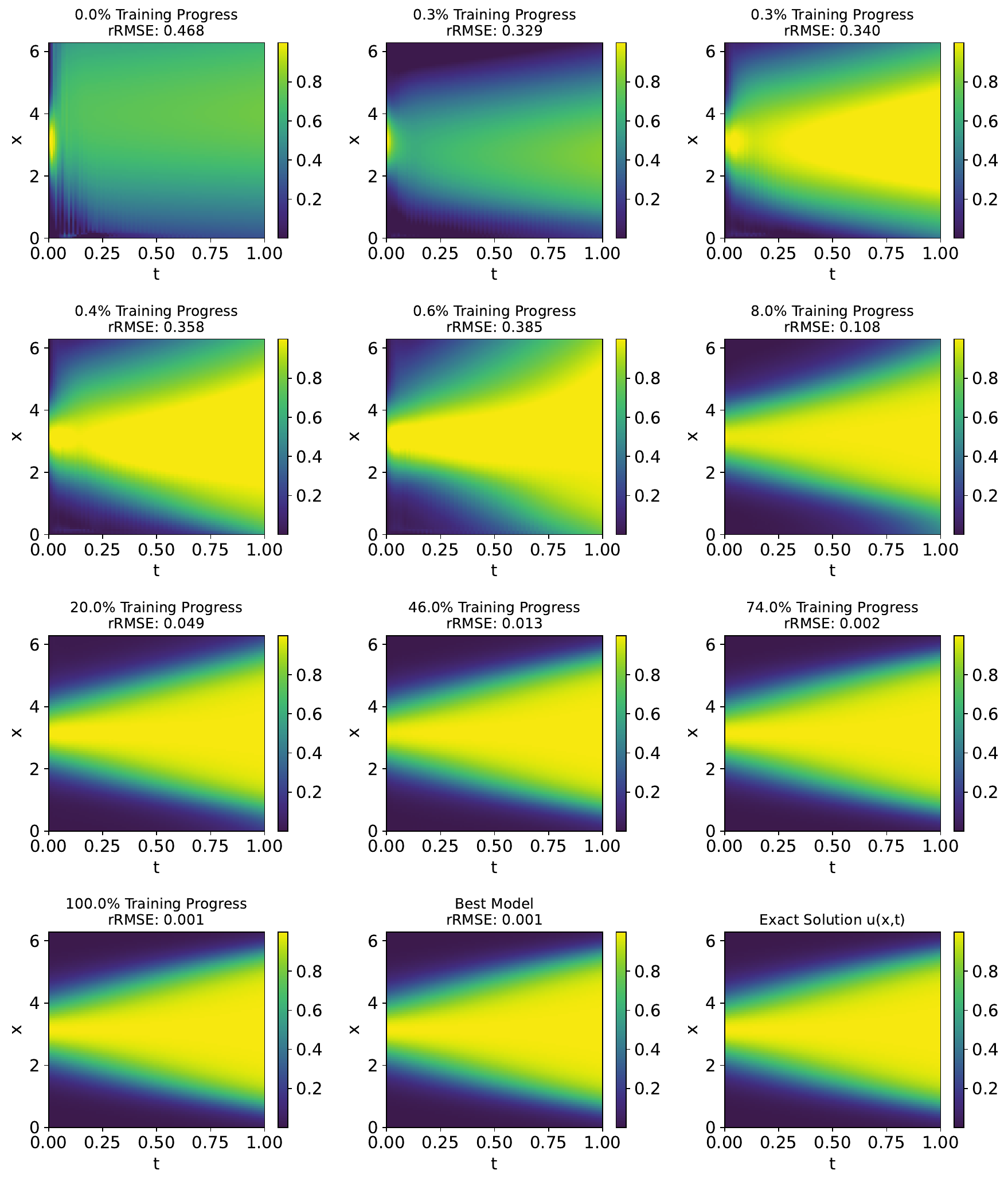}
    \caption{Extended training dynamics of PIANO on the reaction equation. This snapshot grid shows how accurate solution structures first emerge near the initial condition and progressively propagate forward in time as training advances. The gradual refinement across training stages illustrates the model's ability to learn stable temporal evolution through its autoregressive rollout.}
    \label{fig:appendix-reaction-full}
\end{figure*}

\subsubsection{Temporal Profile}

To further analyze the temporal behavior learned by PIANO, we visualize in Figures~\ref{fig:wave_oscillations} and~\ref{fig:convection_temporal_profiles} the predicted time profiles \( u(x_i, t) \) across several fixed spatial locations \( x_i \), at different stages of training for the Wave and Convection equations, respectively. Each subplot corresponds to a particular training progress percentage, with colored curves representing different spatial points.

For the Wave equation, the model initially struggles to capture the high-frequency oscillations, exhibiting distorted amplitudes and poor phase alignment. As training progresses, PIANO gradually learns to reconstruct both the amplitude and frequency content of the wave. Notably, improvements emerge near the initial time and propagate forward, consistent with PIANO’s autoregressive rollout structure. This recursive conditioning enables the model to build the solution step by step, avoiding phase drift and dispersion commonly seen in non-autoregressive methods. By 100\% training, the predicted oscillations closely match the ground truth in amplitude, frequency, and phase across all \( x_i \), indicating that PIANO has successfully learned the global wave dynamics in a stable and physically consistent manner.

The Convection equation presents a different challenge due to its transport-dominated nature. In the early stages of training, predictions are only accurate near the initial time, while further regions suffer from amplitude decay and misaligned phase. However, as training advances, PIANO progressively learns to propagate the sharp waveform forward in time. By conditioning each step on its own prior predictions, the model gradually sharpens the solution and aligns the oscillatory phase across all spatial locations. By the final stages of training, PIANO maintains the structure and timing of the waveform without numerical diffusion, demonstrating its robustness in capturing transport dynamics through stable temporal propagation.

These visualizations confirm that PIANO not only handles oscillatory PDEs like the Wave equation but also excels in transport-heavy regimes like Convection, leveraging its autoregressive architecture to deliver stable and accurate long-term predictions.

\begin{figure*}[ht]
    \centering
    \includegraphics[width=\linewidth]{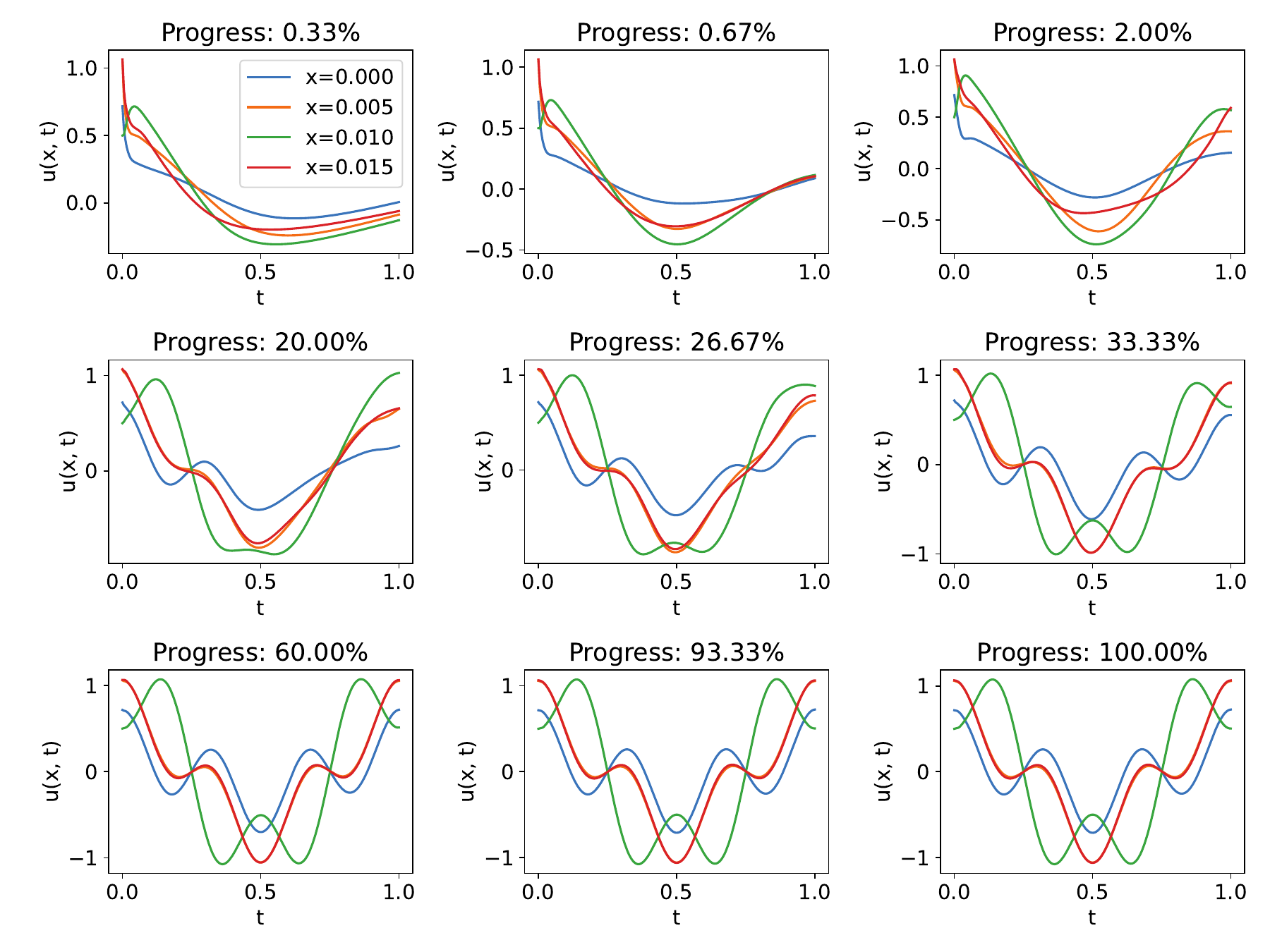}
    \caption{Temporal profile of the Wave equation at various training stages. Each subplot shows the predicted time evolution \( u(x_i, t) \) for multiple spatial points \( x_i \), plotted as separate curves. Early in training (top row), the model captures only coarse low-frequency behavior and underestimates amplitude. As training progresses, PIANO improves its ability to preserve the phase and frequency content of oscillations across all spatial locations. This gradual sharpening of periodic structure demonstrates the autoregressive nature of the model: accurate dynamics emerge near the initial time and propagate forward as the model recursively builds on its own predictions. By the final stages, the predicted waveforms closely match in both phase and amplitude across the entire domain, confirming stable temporal learning.}
    \label{fig:wave_oscillations}
\end{figure*}

\begin{figure*}[ht]
    \centering
    \includegraphics[width=\linewidth]{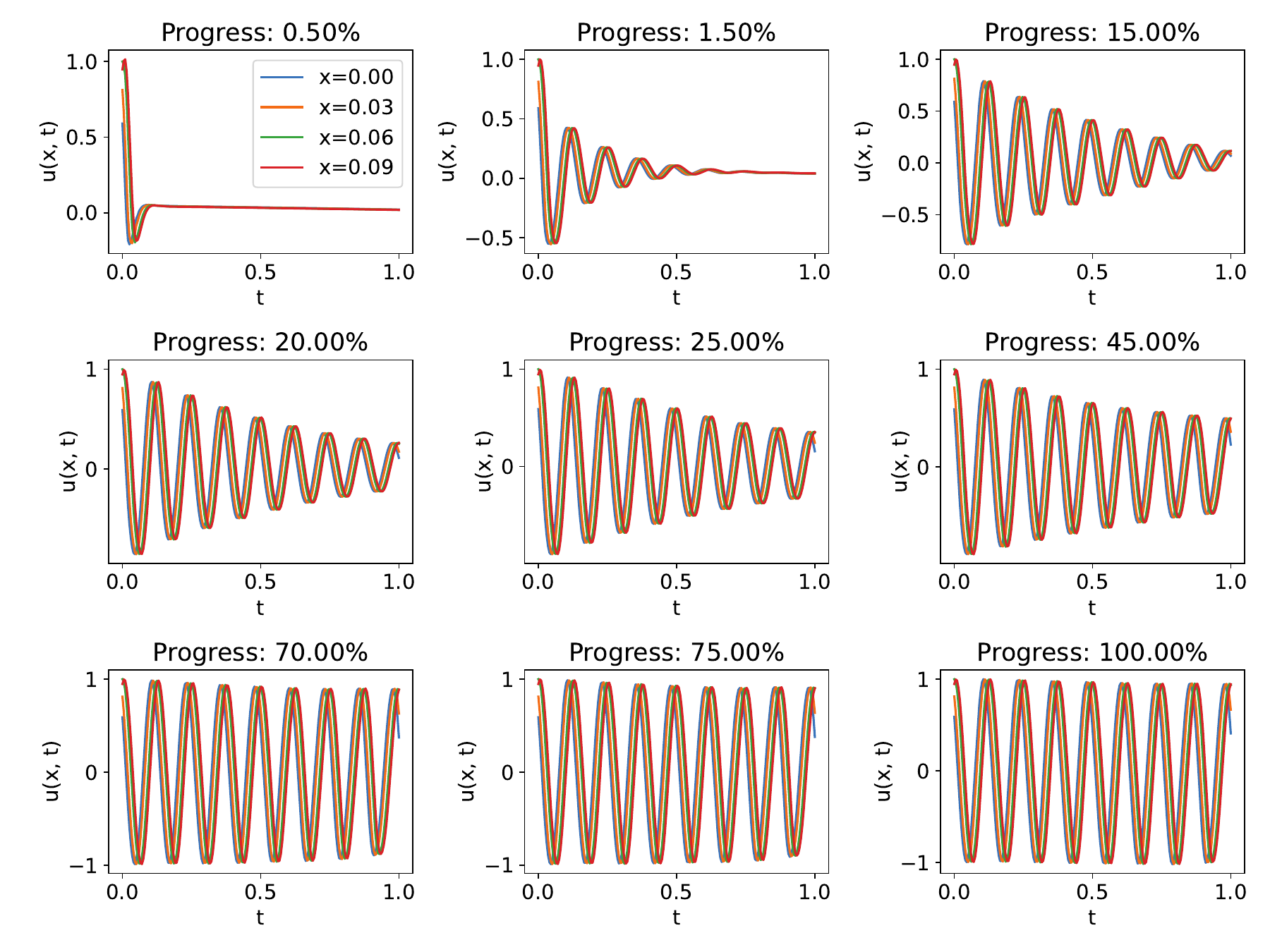}
    \caption{Temporal profiles of the Convection equation at different training stages. Each subplot shows the predicted evolution \( u(x_i, t) \) over time for several fixed spatial locations \( x_i \). Early predictions are accurate only near the initial time, but fail to preserve the wave shape across the domain. As training progresses, PIANO gradually learns to propagate the waveform forward in time, maintaining sharpness and phase alignment. By 100\% training progress, the model stably reproduces the full dynamics, demonstrating robust temporal consistency under transport-dominated behavior.}
    \label{fig:convection_temporal_profiles}
\end{figure*}

\section{Weather Forecasting}
\label{app:weather_forecasting}

\subsection{Evaluation Metrics}

We assess model performance using two standard meteorological metrics: latitude-weighted Root Mean Squared Error (RMSE) and Anomaly Correlation Coefficient (ACC), computed after de-normalizing the predictions.

\begin{equation}
\text{RMSE} = \frac{1}{N} \sum_{t=1}^{N} \left[ \frac{1}{HW} \sum_{h=1}^{H} \sum_{w=1}^{W} \alpha(h) (y_{thw} - u_{thw})^2 \right]^{1/2}
\end{equation}
\begin{equation}
\text{ACC} = \frac{\sum_{t,h,w} \alpha(h)\tilde{y}_{thw}\tilde{u}_{thw}}{\sqrt{\sum_{t,h,w} \alpha(h)\tilde{y}^2_{thw} \sum_{t,h,w} \alpha(h)\tilde{u}^2_{thw}}}
\end{equation}

Here, \( y_{thw} \) and \( u_{thw} \) denote the ground truth and model prediction at time \( t \), latitude index \( h \), and longitude index \( w \), respectively. The term \( \alpha(h) = \cos(h) \big/ \frac{1}{H} \sum_{h'} \cos(h') \) represents the normalized latitude weight, accounting for the area distortion in latitude-longitude grids due to Earth's curvature.

The anomalies are computed by subtracting the empirical mean:
\[
\tilde{y}_{thw} = y_{thw} - C, \quad \tilde{u}_{thw} = u_{thw} - C, \quad 
\]

where $C = \frac{1}{N} \sum_{t} y_{thw}$.

ACC measures the correlation between predicted and true anomalies. Higher ACC indicates better skill in capturing deviations from climatological means. Latitude-weighted RMSE evaluates the spatial accuracy of forecasts while correcting for latitudinal area distortion. Lower RMSE and higher ACC both indicate better forecasting performance.

\subsection{Experimental Setup and Implementation}
Our experimental framework directly mirrors the setup used by ClimODE to ensure a fair comparison. The primary task is forecasting future atmospheric states based on an initial state, with lead times ranging from 6 to 36 hours. The model is implemented in PyTorch. The underlying system of ODEs is solved using the Euler method with a time resolution of 1 hour, managed by the `torchdiffeq` library. All experiments are conducted on a single NVIDIA A100 GPU with 40GB of memory. The model is trained for 300 epochs using a batch size of 8. The learning rate is managed by a Cosine Annealing scheduler.

\subsection{Dataset and Preprocessing}
We use the ERA5 dataset, as preprocessed for the WeatherBench benchmark. The data is provided at a 5.625° spatial resolution with a 6-hour time increment. Our experiments focus on five key variables: 2-meter temperature (\texttt{t2m}), temperature at 850 hPa (\texttt{t}), geopotential at 500 hPa (\texttt{z}), and the 10-meter U and V wind components (\texttt{u10}, \texttt{v10}). All variables are normalized to a [0, 1] range using min-max scaling. The dataset is partitioned by year, with 2006-2015 used for training, 2016 for validation, and 2017-2018 for testing.

\subsection{Results}
We evaluate PIANO's ability to forecast global weather variables using the ERA5 dataset. Figure~\ref{fig:forecast_earth} provides a visual analysis of PIANO’s probabilistic predictions at a fixed forecast time (2017-01-01T12:00) across five key atmospheric variables: geopotential height at 500 hPa (\texttt{z}), temperature at 850 hPa (\texttt{t}), 2-meter surface temperature (\texttt{t2m}), and the 10-meter U and V wind components (\texttt{u10}, \texttt{v10}). Each row corresponds to a variable, while the columns show the predicted mean (\( \mu \)), upper bound (\( \mu + \sigma \)), predicted standard deviation (\( \sigma \)), and pointwise error. These results demonstrate that PIANO not only captures the spatial structure of each variable, but also quantifies predictive uncertainty effectively, with visually low error and consistent uncertainty estimates across regions.

Quantitative results are summarized in Table~\ref{tab:full_results_comparison_weather}, where we report latitude-weighted RMSE and Anomaly Correlation Coefficient (ACC) at multiple forecast lead times, comparing PIANO against several strong neural and numerical baselines, including NODE, ClimaX, FCN, IFS, and ClimODE. Across all variables and lead times, PIANO achieves state-of-the-art performance, often outperforming neural baselines by a significant margin and in some cases approaching the accuracy of IFS. The model shows particularly strong gains in mid-range horizons (12–24 hours), maintaining high correlation and low error while producing calibrated uncertainty estimates. These results highlight the benefit of PIANO’s autoregressive, physics-informed structure for long-range, high-resolution weather modeling.

\begin{figure*}
    \centering
    \includegraphics[width=\linewidth]{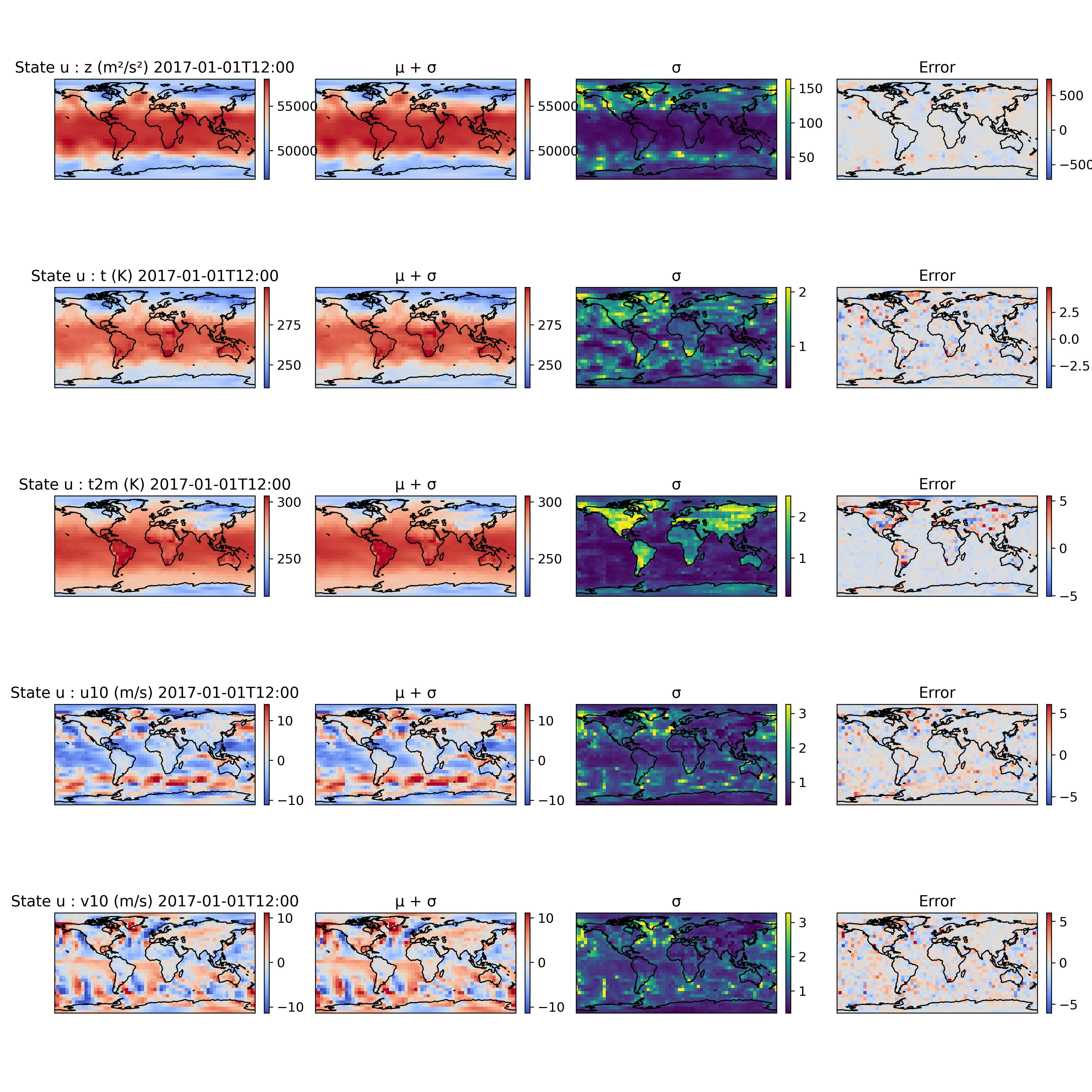}
    \caption{Visualization of PIANO’s forecasting capabilities for five atmospheric state variables on 2017-01-01 at 12:00. Each row corresponds to a different variable: geopotential (\texttt{z}), atmospheric temperature (\texttt{t}), 2-meter surface temperature (\texttt{t2m}), and the 10-meter U-wind (\texttt{u10}) and V-wind (\texttt{v10}) components. The columns show (from left to right): the predicted mean \( \mu \), upper bound \( \mu + \sigma \), predicted standard deviation \( \sigma \), and prediction error (difference from ground truth). PIANO captures spatial structure and uncertainty across variables, with low errors. The errors are more pronounced where PIANO already suggests uncertainty ($\sigma$).}
    \label{fig:forecast_earth}
\end{figure*}

\begin{table*}
\resizebox{\textwidth}{!}{%
\begin{tabular}{llcccccc|cccccc}
\toprule
& & \multicolumn{6}{c}{\textbf{RMSE($\downarrow$)}} & \multicolumn{6}{c}{\textbf{ACC($\uparrow$)}} \\
\cmidrule(lr){3-8} \cmidrule(lr){9-14}
\textbf{Variable} & \textbf{Lead-Time (hrs)} & \textbf{NODE} & \textbf{ClimaX} & \textbf{FCN} & \textbf{IFS} & \textbf{ClimODE} & \textbf{PIANO (Ours)} & \textbf{NODE} & \textbf{ClimaX} & \textbf{FCN} & \textbf{IFS} & \textbf{ClimODE} & \textbf{PIANO (Ours)} \\
\midrule
\multirow{5}{*}{\textbf{z}} 
 & 6 & 300.64 & 247.5 & 149.4 & 26.9 & $102.9 \pm 9.3$ & $69.07 \pm 4.99$ & 0.96 & 0.97 & 0.99 & 1.00 & 0.99 & 1.00 \\
 & 12 & 460.23 & 265.4 & 217.8 & (N/A) & $134.8 \pm 12.3$ & $109.07 \pm 8.30$ & 0.88 & 0.96 & 0.99 & (N/A) & 0.99 & 0.99 \\
 & 18 & 627.65 & 319.8 & 275.0 & (N/A) & $162.7 \pm 14.4$ & $145.99 \pm 11.95$ & 0.79 & 0.95 & 0.99 & (N/A) & 0.98 & 0.99 \\
 & 24 & 877.82 & 364.9 & 333.0 & 51.0 & $193.4 \pm 16.3$ & $185.22 \pm 15.91$ & 0.70 & 0.93 & 0.99 & 1.00 & 0.98 & 0.98 \\
 & 36 & 1028.20 & 455.0 & 449.0 & (N/A) & $259.6 \pm 22.3$ & $263.44 \pm 22.96$ & 0.55 & 0.89 & 0.99 & (N/A) & 0.96 & 0.97 \\
\midrule
\multirow{5}{*}{\textbf{t}} 
 & 6 & 1.82 & 1.64 & 1.18 & 0.69 & $1.16 \pm 0.06$ & $0.92 \pm 0.04$ & 0.94 & 0.94 & 0.99 & 0.99 & 0.97 & 0.98 \\
 & 12 & 2.32 & 1.77 & 1.47 & (N/A) & $1.32 \pm 0.13$ & $1.16 \pm 0.05$ & 0.85 & 0.93 & 0.99 & (N/A) & 0.96 & 0.97 \\
 & 18 & 2.93 & 1.93 & 1.65 & (N/A) & $1.47 \pm 0.16$ & $1.32 \pm 0.06$ & 0.77 & 0.92 & 0.99 & (N/A) & 0.96 & 0.96 \\
 & 24 & 3.35 & 2.17 & 1.83 & 0.87 & $1.55 \pm 0.18$ & $1.48 \pm 0.07$ & 0.72 & 0.90 & 0.99 & 0.99 & 0.95 & 0.96 \\
 & 36 & 4.13 & 2.49 & 2.21 & (N/A) & $1.75 \pm 0.26$ & $1.76 \pm 0.09$ & 0.58 & 0.86 & 0.99 & (N/A) & 0.94 & 0.94 \\
\midrule
\multirow{5}{*}{\textbf{t2m}} 
 & 6 & 2.72 & 2.02 & 1.28 & 0.97 & $1.21 \pm 0.09$ & $1.01 \pm 0.05$ & 0.82 & 0.92 & 0.99 & 0.99 & 0.97 & 0.98 \\
 & 12 & 3.16 & 2.26 & 1.48 & (N/A) & $1.45 \pm 0.10$ & $1.20 \pm 0.09$ & 0.68 & 0.90 & 0.99 & (N/A) & 0.96 & 0.97 \\
 & 18 & 3.45 & 2.45 & 1.61 & (N/A) & $1.43 \pm 0.09$ & $1.29 \pm 0.08$ & 0.69 & 0.88 & 0.99 & (N/A) & 0.96 & 0.97 \\
 & 24 & 3.86 & 2.37 & 1.68 & 1.02 & $1.40 \pm 0.09$ & $1.42 \pm 0.10$ & 0.79 & 0.89 & 0.99 & 0.99 & 0.96 & 0.96 \\
 & 36 & 4.17 & 2.87 & 1.90 & (N/A) & $1.70 \pm 0.15$ & $1.68 \pm 0.15$ & 0.49 & 0.83 & 0.99 & (N/A) & 0.94 & 0.94 \\
\midrule
\multirow{5}{*}{\textbf{u10}} 
 & 6 & 2.30 & 1.58 & 1.47 & 0.80 & $1.41 \pm 0.07$ & $1.24 \pm 0.06$ & 0.85 & 0.92 & 0.95 & 0.98 & 0.91 & 0.95 \\
 & 12 & 3.13 & 1.96 & 1.89 & (N/A) & $1.81 \pm 0.09$ & $1.53 \pm 0.07$ & 0.70 & 0.88 & 0.93 & (N/A) & 0.89 & 0.93 \\
 & 18 & 3.41 & 2.24 & 2.05 & (N/A) & $1.97 \pm 0.11$ & $1.74 \pm 0.07$ & 0.58 & 0.84 & 0.91 & (N/A) & 0.88 & 0.91 \\
 & 24 & 4.10 & 2.49 & 2.33 & 1.11 & $2.01 \pm 0.10$ & $1.96 \pm 0.09$ & 0.50 & 0.80 & 0.89 & 0.97 & 0.87 & 0.88 \\
 & 36 & 4.68 & 2.98 & 2.87 & (N/A) & $2.25 \pm 0.18$ & $2.35 \pm 0.12$ & 0.35 & 0.69 & 0.85 & (N/A) & 0.83 & 0.83 \\
\midrule
\multirow{5}{*}{\textbf{v10}} 
 & 6 & 2.58 & 1.60 & 1.54 & 0.94 & $1.53 \pm 0.08$ & $1.30 \pm 0.06$ & 0.81 & 0.92 & 0.94 & 0.98 & 0.92 & 0.95 \\
 & 12 & 3.19 & 1.97 & 1.81 & (N/A) & $1.81 \pm 0.12$ & $1.58 \pm 0.07$ & 0.61 & 0.88 & 0.91 & (N/A) & 0.89 & 0.92 \\
 & 18 & 3.58 & 2.26 & 2.11 & (N/A) & $1.96 \pm 0.16$ & $1.79 \pm 0.08$ & 0.46 & 0.83 & 0.86 & (N/A) & 0.88 & 0.90 \\
 & 24 & 4.07 & 2.48 & 2.39 & 1.33 & $2.04 \pm 0.10$ & $2.01 \pm 0.09$ & 0.35 & 0.80 & 0.83 & 0.97 & 0.86 & 0.88 \\
 & 36 & 4.52 & 2.98 & 2.95 & (N/A) & $2.29 \pm 0.24$ & $2.40 \pm 0.13$ & 0.29 & 0.69 & 0.75 & (N/A) & 0.83 & 0.82 \\
\bottomrule
\end{tabular}}
\caption{Latitude weighted RMSE($\downarrow$) and Anomaly Correlation Coefficient (ACC($\uparrow$)) comparison with baselines on global forecasting on the ERA5 dataset. PIANO generally outperforms all the neural baselines.}
\label{tab:full_results_comparison_weather}
\end{table*}
\end{document}